\documentclass[lettersize,journal]{IEEEtran}
\usepackage{amsmath,amsfonts}
\usepackage{array}
\usepackage{textcomp}
\usepackage{stfloats}
\usepackage{verbatim}

\usepackage{mathtools}
\usepackage{amssymb}
\usepackage{listings}
\usepackage{xcolor}
\definecolor{codegreen}{rgb}{0,0.6,0}
\definecolor{codegray}{rgb}{0.5,0.5,0.5}
\definecolor{codepurple}{rgb}{0.58,0,0.82}
\definecolor{backcolour}{rgb}{0.95,0.95,0.92}

\newcommand{\blue}[1]{\begin{color}{black}#1\end{color}}

\lstdefinestyle{mystyle}{
    backgroundcolor=\color{backcolour},   
    commentstyle=\color{codegreen},
    keywordstyle=\color{magenta},
    numberstyle=\tiny\color{codegray},
    stringstyle=\color{codepurple},
    basicstyle=\ttfamily\footnotesize,
    breakatwhitespace=false,         
    breaklines=true,                 
    captionpos=b,                    
    keepspaces=true,                 
    numbers=left,                    
    numbersep=5pt,                  
    showspaces=false,                
    showstringspaces=false,
    showtabs=false,                  
    tabsize=2
}

\usepackage{wrapfig, blindtext}
\definecolor{dodgerblue}{rgb}{0.12, 0.56, 1.0}
\definecolor{gold}{rgb}{1.0, 0.84, 0.0}

\usepackage[utf8]{inputenc} 
\usepackage[T1]{fontenc}    
\usepackage{hyperref}       
\usepackage{url}            
\usepackage{booktabs}       
\usepackage{amsfonts}       
\usepackage{nicefrac}       
\usepackage{microtype}      
\usepackage[ruled,vlined]{algorithm2e}
\usepackage{graphicx}
\usepackage{multirow}
\usepackage{subfigure}
\usepackage{amsmath}

\usepackage{enumitem}
\usepackage{amsthm}
\newtheorem{theorem}{Theorem}

\newtheorem{lemma}{Lemma}
\newtheorem{definition}{Definition}
\newtheorem{assumption}{Assumption}
\newtheorem{remark}{Remark}

\newtheorem{hypothesis}{Hypothesis}
\usepackage{bm}
\usepackage[numbers,sort]{natbib}

\begin{document}

\title{Understanding Deep Learning via Decision Boundary}

\author{Shiye Lei, Fengxiang He, Yancheng Yuan, Dacheng Tao,~\IEEEmembership{Fellow,~IEEE}
\thanks{S. Lei and D. Tao are with the Sydney AI Centre and School of Computer Science, Faculty of Engineering, The University of Sydney. Email: slei5230@uni.sydney.edu.au and dacheng.tao@sydney.edu.au.}%
\thanks{F. He is with Artificial Intelligence and its Applications Institute, School of Informatics, University of Edinburgh. E-mail: F.He@ed.ac.uk.}%
\thanks{Y. Yuan is with Department of Applied Mathematics, The Hong Kong Polytechnic University. Email: yancheng.yuan@polyu.edu.hk.}%
\thanks{Corresponding author: F. He and D. Tao.}
}



\maketitle

\begin{abstract}
This paper discovers that the neural network with lower decision boundary (DB) variability has better generalizability. Two new notions, {\it algorithm DB variability} and {\it $(\epsilon, \eta)$-data DB variability}, are proposed to measure the decision boundary variability from the algorithm and data perspectives. Extensive experiments show significant negative correlations between the decision boundary variability and the generalizability. From the theoretical view, two lower bounds based on algorithm DB variability are proposed and do not explicitly depend on the sample size. We also prove an upper bound of order $\mathcal{O}\left(\frac{1}{\sqrt{m}}+\epsilon+\eta\log\frac{1}{\eta}\right)$ based on data DB variability. The bound is convenient to estimate without the requirement of labels, and does not explicitly depend on the network size which is usually prohibitively large in deep learning.
\end{abstract}


\section{Introduction}
\IEEEPARstart{N}{eural} networks (NNs) have achieved significant success in vast applications \citep{krizhevsky2012imagenet,vaswani2017attention}, including computer vision \cite{he2016deep}, natural language processing \cite{brown2020language}, and data mining \cite{wang2015collaborative}. However, the advance of NNs is arduous to be characterized by conventional statistical learning theory based on hypothesis complexity \cite{mohri2018foundations}, such as VC-dimension \citep{vapnik1994measuring} and Rademacher complexity \citep{bartlett2002rademacher}. According to the conventional theory, models of larger hypothesis complexity possess worse generalizability, while neural networks are usually over-paramterized but have excellent generalizability. 

In this paper, we attempt to explain the excellent generalizability of deep learning from the perspective of decision boundary (DB) variability. 
Intuitively, the decision boundary variability of a neural networks is largely determined by two means: (1) the randomness introduced by the optimization algorithm, and (2) the fluctuations of the training data when they are sampled from the data generating distribution. Following this intuition, we design two terms, {\it algorithm DB variability} and {\it $(\epsilon,\eta)$-data DB variability}, to measure the DB variability. 

\begin{figure}[t]
\centering
\subfigure[Algorithm DB variability]{
\begin{minipage}[b]{0.22\textwidth}
\centering
    		\includegraphics[width=1.\columnwidth]{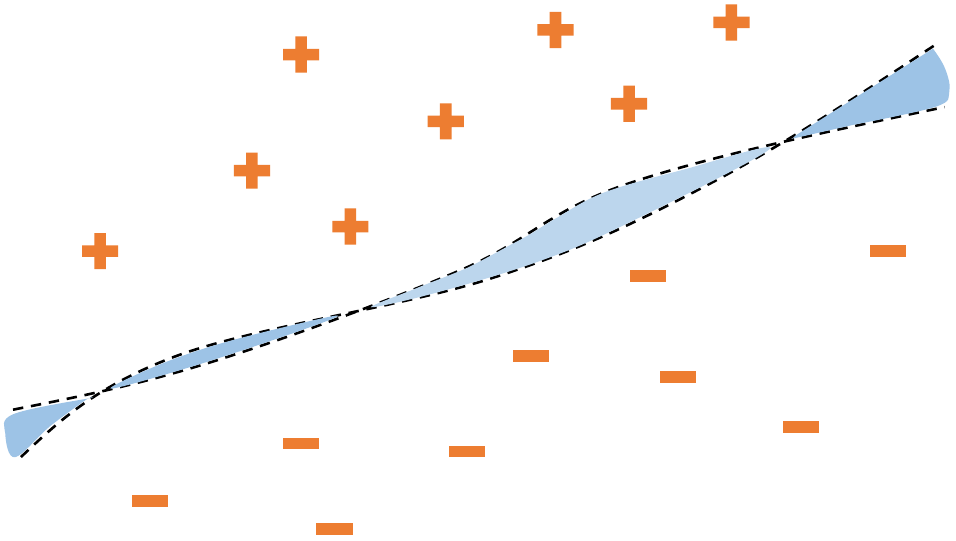}
    		\end{minipage}
		\label{figure:illustration algorithm DB}   
    	}
\subfigure[Data DB variablity]{
\begin{minipage}[b]{0.22\textwidth}
\centering
    		\includegraphics[width=1.\columnwidth]{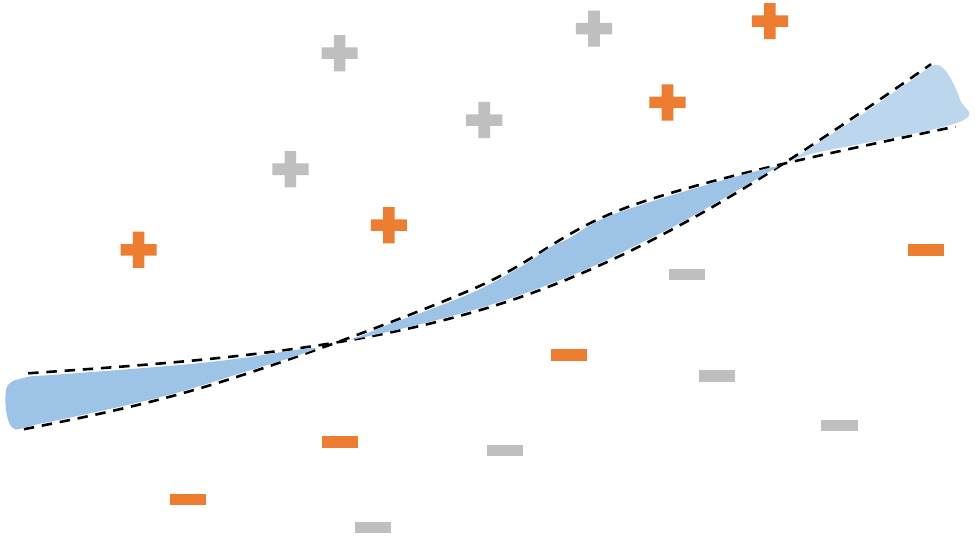}
    		\end{minipage}
		\label{figure:illustration data DB}   
    	}
\caption{An illustration of decision boundary variability. (a) Two dashed curves denote the decision boundaries of the model trained on the orange points with different repeats, respectively, and the blue mismatch region is connected to the algorithm DB variability; (b) Two dashed curves denote the decision boundaries of the model trained on the orange points and all (orange and gray) points, respectively, and the blue mismatch region is connected to the data DB variability.}
\label{figure:illustration}
\end{figure}

\textbf{Algorithm DB variability} measures the variability of DBs in different training repeats. We conduct extensive experiments on the CIFAR-10/100 datasets \citep{krizhevsky2009learning} to explore which factors would determine algorithm DB variability. We visualize the trend of the algorithm DB variability with respect to ({\it w.r.t.}) training strategies, training time, sample sizes, and label noise ratios. The empirical results demonstrate that algorithm DB variability has (1) negative correlations with the training time and the sample size, (2) a positive correlation with the label noise, and (3) a negative correlation with the generalizability (or, test accuracy, in experiments). 
{Benefiting from this significant correlation, algorithm DB variability can be employed to do model selection without the requirement of test labels, and shows superior performance than conventional marginal likelihood measurement.} From the theoretical view, we prove two lower bounds based on the algorithm DB variability, which fully support our experiments. 


\textbf{$(\epsilon,\eta)$-data DB variability} is proposed to characterize the decision boundary from the view of the randomness in training data. 
Given a neural network, if its decision boundary can be ``reconstructed'' by training a network with the same architecture from the scratch on a smaller $\eta$-subset (which contain $\eta\%$ examples of the source training set), while the ``error'' of the reconstruction is not larger than $\epsilon$, we call the model has $(\epsilon,\eta)$-data DB variability. Specifically, we may define the reconstruction error as the approximation error of the reconstructed decision boundary on the whole training set. Moreover, an $\eta$-$\epsilon$ curve can be drawn numerically. 
The area under the $\eta$-$\epsilon$ curve could be an informative indicator for characterizing the generalization of NNs. 
An $\mathcal{O}\left(\frac{1}{\sqrt{m}}+\epsilon+\eta\log\frac{1}{\eta}\right)$ generalization bound based on the $(\epsilon, \eta)$-data DB variability is proved, which demonstrates the relationship between the generalization of NNs and DB variability. 
In contrast to many existing generalization bounds based on hypothesis complexity that require the access to the weight norm, our bounds only need the predictions. This brings significant advantages in empirically approximating the generalization bound in (1) black-model settings, where model parameters are unavailable, and (2) over-parameterized settings, where calculating the weight norm is of a prohibitively high computing burden.

To our best knowledge, this is the first work on explaining deep learning via the variability of decision boundary. Our research also sheds light to understanding a variety of interesting phenomenons, including the entropy and the complexity of decision boundary. 
Through the lens of decision boundary variability, we may also design novel algorithms via reducing the decision boundary variability.




\section{Related works}
\textbf{Deep learning theory.} 
In learning theory, generalizability refers to the capability of well-trained models predicting on unseen data. Conventional theory suggests that the generalizability has a negative correlation with the hypothesis complexity \citep{mohri2018foundations}, such as VC-dimension \citep{vapnik1994measuring} and Rademacher complexity \citep{bartlett2002rademacher}: models with larger complexity fit the training data better. This is usually summarised as the ``bias-variance trade-off''. This principle faces significant challenges in deep learning \citep{ma2020towards,he2020recent}. 
\citet{zhang2021understanding} demonstrate that neural networks can near-perfectly fit noisy labels (which suggests that deep learning has an extremely large Rademarcher complexity), but still have impressive generalization performance. This conflict draws attentions of numerous researchers \citep{belkin2019reconciling,nakkiran2019deep,9257392, chen2023spectral}. \citet{belkin2019reconciling} show the unusual double descent phenomenon of the training error {\it w.r.t.} model size following by works \citep{nakkiran2019deep,li2020benign}, which further sheds shadows to the ``bias-variance trade-off''.

Many works attribute the success of neural networks to the effectiveness of the stochastic gradient descent (SGD) algorithm \citep{bottou2010large,hardt2016train,he2019control,9222567}. For example, {\citet{jin2017escape} show that} SGD can escape from the local minima. The loss landscape of the networks has also be extensively analysed and it has been proven that there is no spurious local minima for linear NNs \citep{kawaguchi2016deep,lu2017depth,zhou2017critical}. Nevertheless, this elegant property does not hold for general networks where non-linear activation functions are involved \citep{he2020piecewise,Goldblum2020Truth}. 
Recent study also attempts to explore the implicit bias of neural networks in the over-parameterized regime.  \citet{soudry2018implicit} show that the over-parameterized networks converge to the max-margin solution when the training data is linear-separable. Some other research has also been conducted along this line \citep{NEURIPS2020_c76e4b2f,Lyu2020Gradient,chizat2020implicit}. 

Empirical studies have also attempted to explain the decent performance of networks by uncovering their learning properties \citep{nakkiran2020deep,jiang2021assessing,9319542,lei2021understanding}. For instance, neural networks are shown tend to fit the low-frequency information first \citep{rahaman2019spectral,xu2019frequency} and then gradually learn more complex patterns \citep{kalimeris2019sgd} during the training procedure. \citet{he2020local} show that neural networks own the unique property of local elasticity that the predictions on the input data $\mathbf{x}'$ will not be significantly perturbed, when the neural net is updated via SGD at the training example $\mathbf{x}$ if $\mathbf{x}'$ is ``dissimilar'' to $\mathbf{x}$. Similar phenomena are also observed by \citet{fort2019stiffness}. Besides, \citet{papyan2020prevalence} uncover a novel phenomenon, neural collapse, which sheds light on interpreting the effectiveness of deep models \citep{fang2021exploring}.

\textbf{Decision boundaries in neural networks.} Decision boundary, which partitions the input space with different labels, is an important notion in machine learning. Recent studies attempt to understand neural networks from the aspect of decision boundaries \citep{he2018decision,karimi2019characterizing,karimi2020decision}. \citet{alfarra2020on} employ the tropical geometry to represent the decision boundary of neural networks. \citet{guan2020analysis} empirically show a negative correlation between the complexity of decision boundary and the generalization performance of neural networks. \citet{mickisch2020understanding} reveal an insightful phenomenon that the distance from data to the decision boundary continuously decreases during the training. More recently, researchers uncover that neural networks only rely on the most discriminative or the simplest features to construct the decision boundary \citep{ortiz2020hold,shah2020pitfalls}. Besides, \citet{samangouei2018explaingan} also explain the predictions of neural networks  via constructing examples crossing the decision boundary. 
To our best knowledge, this paper is the first work on (1) theoretically characterizing the complexity of decision boundary via the new measure, decision boundary variability, and (2) explaining  the negative correlation between the generalizability and decision boundary variability.

\begin{figure*}[t]
\centering
\subfigure[Fake CIFAR-10]{
\begin{minipage}[b]{0.22\textwidth}
\centering
    		\includegraphics[width=1.\columnwidth]{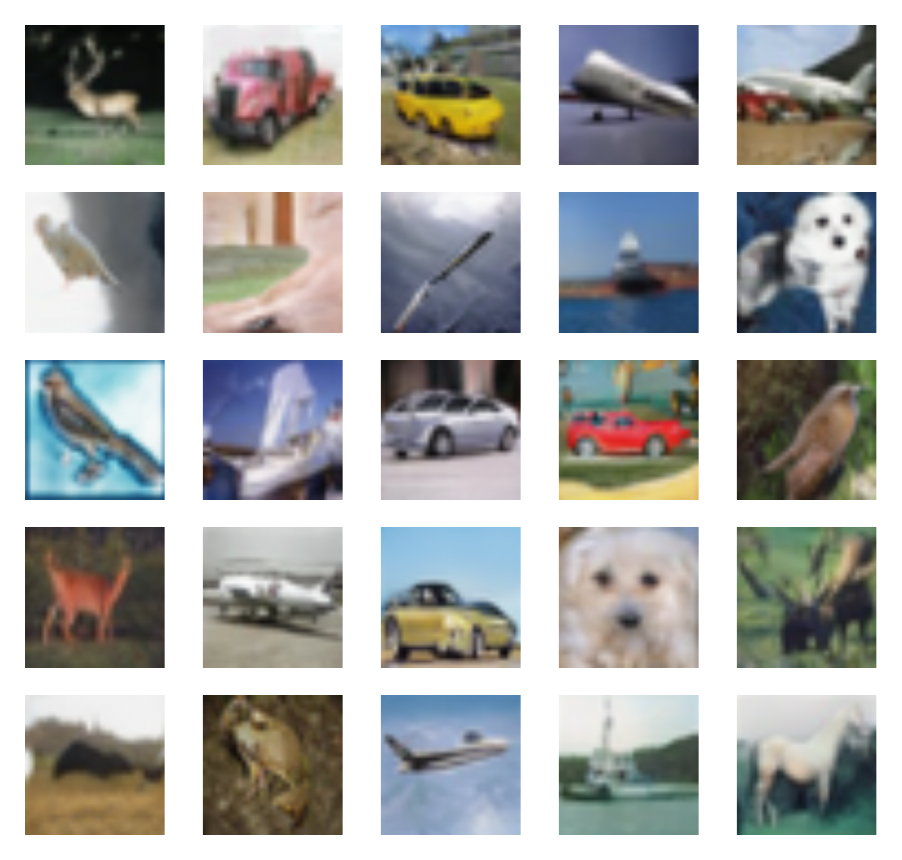}
    		\end{minipage}
		\label{figure:biggan_cifar10}
    	}
\subfigure[Fake CIFAR-100]{
\begin{minipage}[b]{0.22\textwidth}
\centering
    		\includegraphics[width=1.\columnwidth]{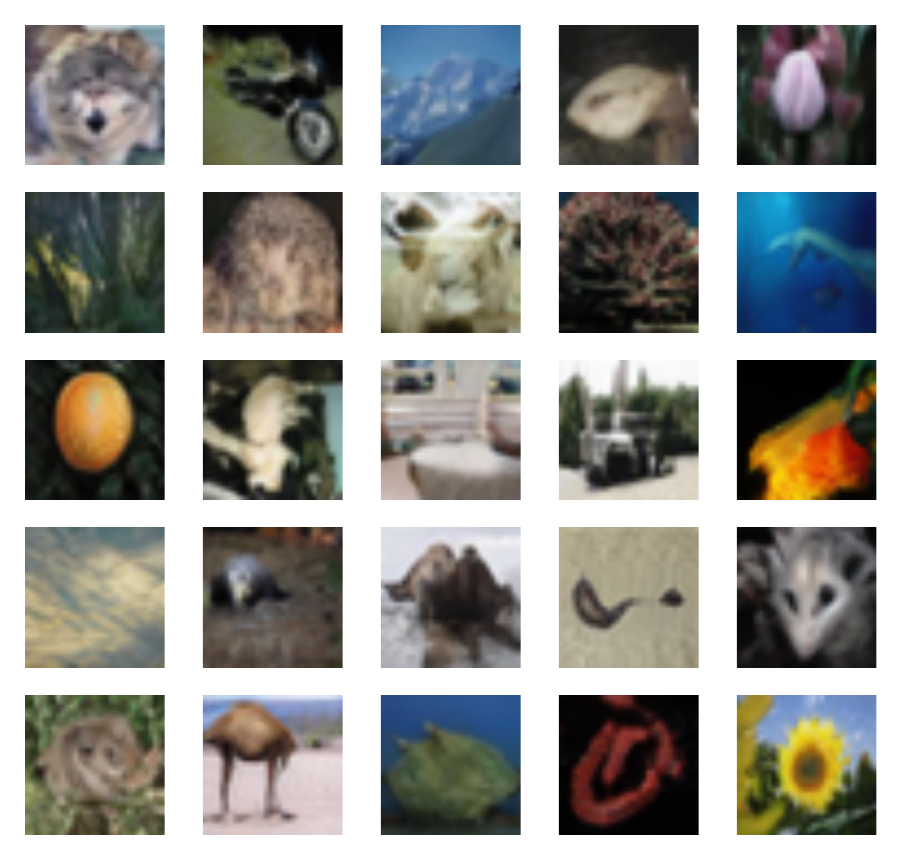}
    		\end{minipage}
		\label{figure:biggan_cifar100}
    	}
\subfigure[Algorithm DB variability vs. test accuracy]{
\begin{minipage}[b]{0.46\textwidth}
            \centering
    		\includegraphics[width=0.48\columnwidth]{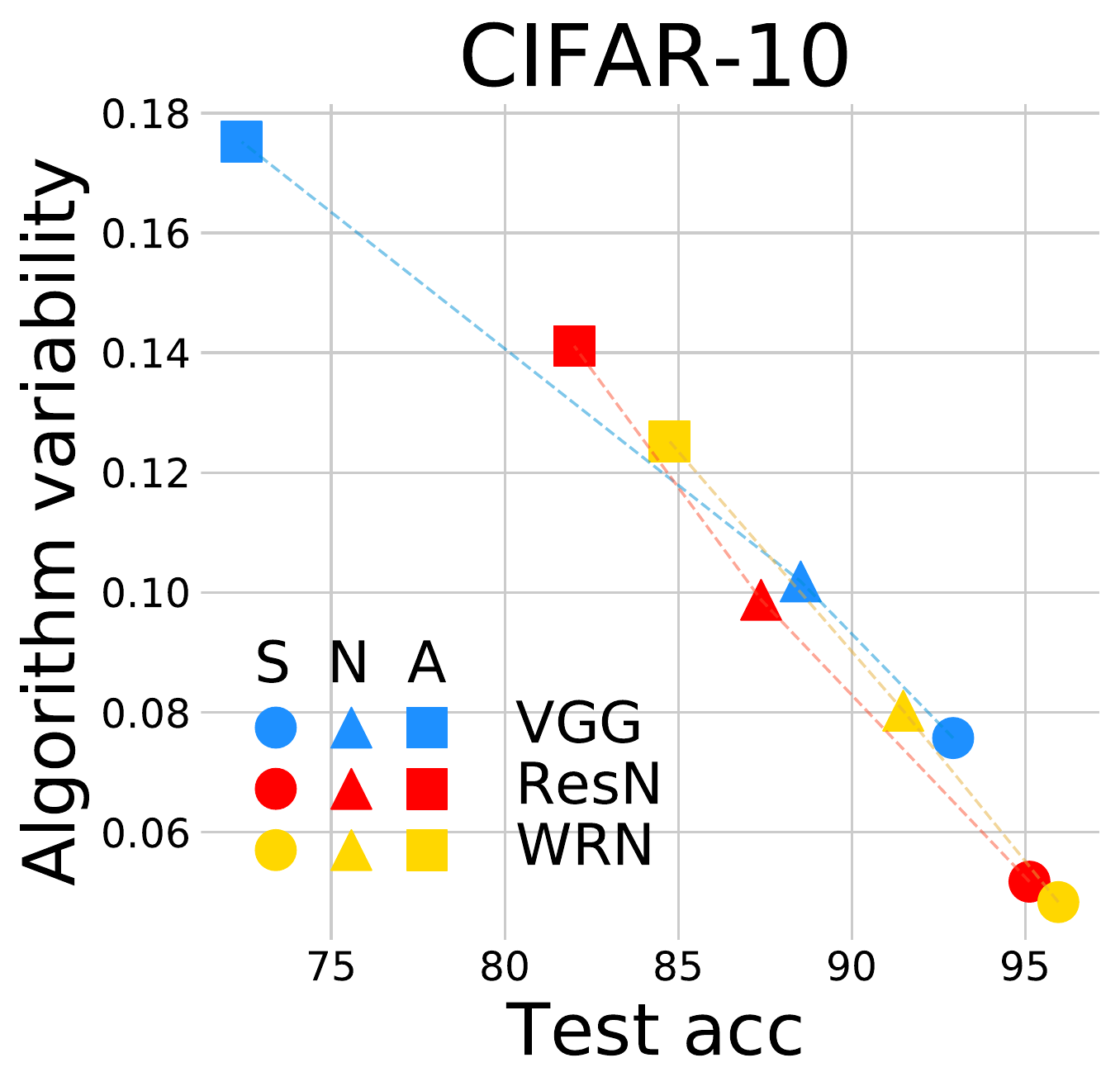}
    		\includegraphics[width=0.48\columnwidth]{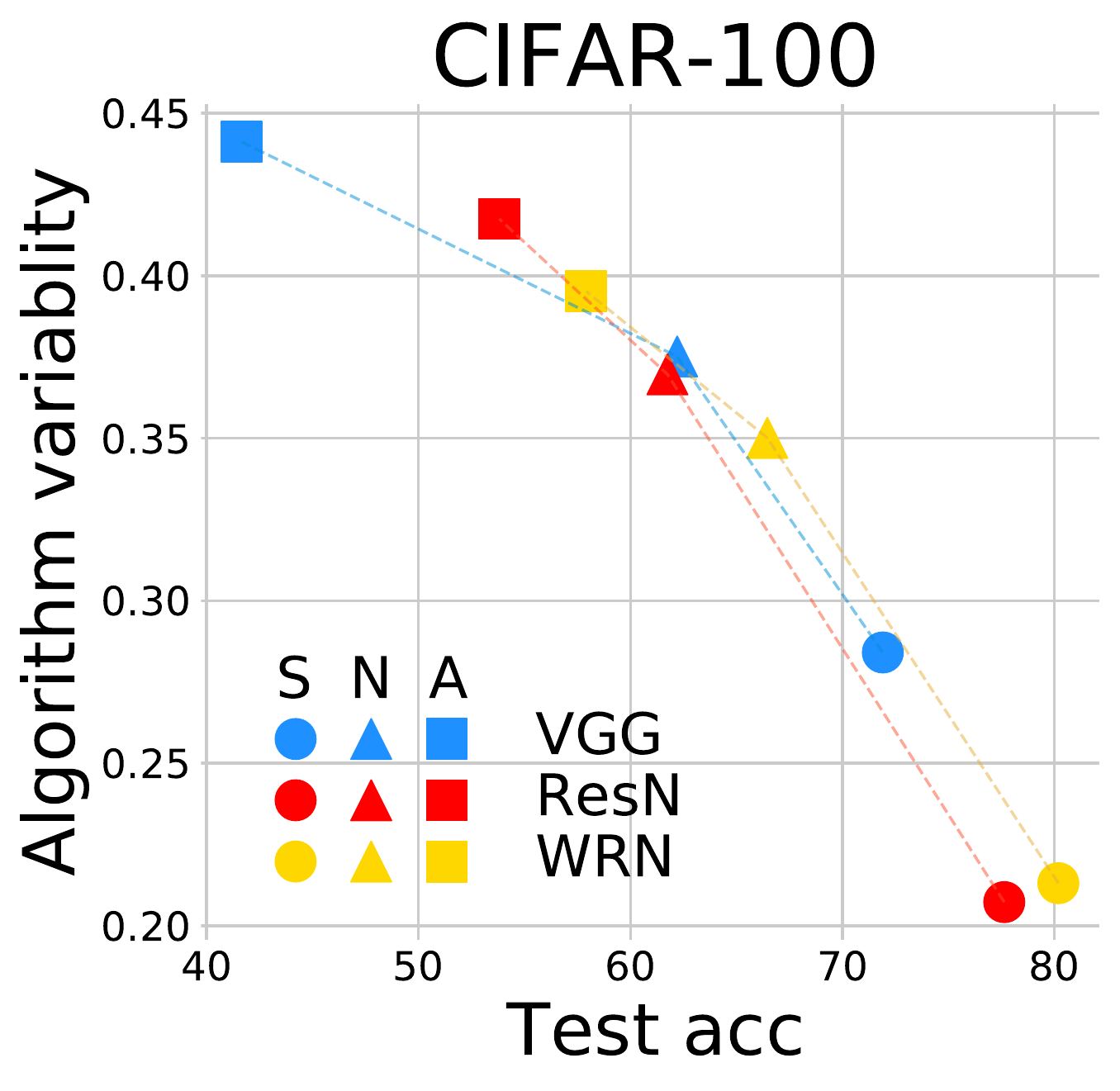}
    		\end{minipage}
		\label{figure:result_cifar}
    	}
\caption{Algorithm decision boundary variability on CIFAR-10 and CIFAR-100. (a) Examples of fake CIFAR-10 images generated by conditional BigGAN. (b) Examples of fake CIFAR-100 images generated by conditional BigGAN. 
(c) Scatter plots of algorithm DB variability to accuracy on test set with different architectures and training strategies on CIFAR-10 and CIFAR-100. The colors of {\color{dodgerblue}blue}, {\color{red}red}, and {\color{gold}yellow} points denote the architectures of VGG-16 (VGG), ResNet-18 (ResN), and WideResNet-28 (WRN), respectively. The shapes of \raisebox{-0.7pt}{\Large $\bullet$}    , {\normalsize $\blacktriangle$}, and {\small $\blacksquare$} designate the training strategies of standard training (S), non-data-augmentation training (N), and adversarial training (A), respectively. Each point is calculated and then averaged on $10$ trials.}
\label{figure:read-word}      
\end{figure*}

\section{Preliminaries}
We denote the training set by $\mathcal{S}=\{(\mathbf{x}_i, y_i), i=1,\dots,m\}$, where $\mathbf{x}_i\in \mathbb{R}^n$, $n$ is the dimension of input data, $y_i\in [k] = \{1,\dots,k\}$, $k$ is the number of classes, and $m=|\mathcal{S}|$ is the training sample size. 
We assume that $(\mathbf{x}_i, y_i)$ are independent and identically distributed (i.i.d.) random variables drawn from the data generating distribution $\mathcal{D}$. {Denote the classifier as} $f_{\boldsymbol{\theta}}(\mathbf{x}): \mathbb{R}^n \rightarrow \mathbb{R}^k$, {which} is a neural network parameterized by $\boldsymbol{\theta}$. The output of $f_{\boldsymbol{\theta}}(\mathbf{x})$ is a $k$-dimensional vector and {is assumed to be a discrete probability density function}. Let $f^{(i)}_{\boldsymbol{\theta}}(\mathbf{x})$ {be} the $i$-th component of $f_{\boldsymbol{\theta}}(\mathbf{x})$, hence $\sum^k_{i=1}f^{(i)}_{\boldsymbol{\theta}}(\mathbf{x}) = 1$. We define $T(f_{\boldsymbol{\theta}}, \mathbf{x})=\{i\in \{1,\cdots,k\}| f^{(i)}_{\boldsymbol{\theta}}(\mathbf{x})=\max_{j} f^{(j)}_{\boldsymbol{\theta}}(\mathbf{x}) \}$ to denote the set of predicted labels by $f_{\boldsymbol{\theta}}$ on $\mathbf{x}$.
Due to the randomness of the learning algorithm $\mathcal{A}$, let $\mathbb{Q}(\boldsymbol{\theta})=\mathcal{A}(\mathcal{S})$ denote the posteriori distribution returned by the learning algorithm $\mathcal{A}$ leveraged on the training set $\mathcal{S}$. Hence, we focus on the {\it Gibbs classifier} (a.k.a. random classifier) $f_\mathbb{Q}=\{f_{\boldsymbol{\theta}} \vert \boldsymbol{\theta}\sim \mathbb{Q}\}$. $0-1$ loss is employed in this paper, and the expected risks in terms of $\boldsymbol{\theta}$ and $\mathbb{Q}$ are defined as:
\begin{equation}
    \mathcal{R}_\mathcal{D}(\boldsymbol{\theta}) = \mathbb{E}_{(\mathbf{x},y)\sim \mathcal{D}} \left[\mathbb{I}\left(y\notin T\left(f_{\boldsymbol{\theta}}, \mathbf{x}\right)\right)\right]
\end{equation}
and
\begin{equation}
    \mathcal{R}_\mathcal{D}(\mathbb{Q}) = \mathbb{E}_{(\mathbf{x},y)\sim \mathcal{D}}\mathbb{E}_{\boldsymbol{\theta}\sim \mathbb{Q}} \left[\mathbb{I}\left(y\notin T\left(f_{\boldsymbol{\theta}}, \mathbf{x}\right)\right)\right],
\end{equation}
respectively. {Here} $\mathbb{I}(\cdot)$ is the indicator function. {Since the data generating distribution $\mathcal{D}$ is unknown}, {evaluating} the expected risk $\mathcal{R}_\mathcal{D}$ is not practical. Therefore, it is a practical way to estimate the expected risk by the empirical risk $\mathcal{R}_\mathcal{S}$, which is defined as:
\begin{align}
    \mathcal{R}_\mathcal{S}(\boldsymbol{\theta}) = & \mathbb{E}_{(\mathbf{x},y)\sim \mathcal{S}}\left[\mathbb{I}\left(y\notin T\left(f_{\boldsymbol{\theta}}, \mathbf{x}\right)\right)\right] \nonumber \\
    = & \frac{1}{m}\sum_{i=1}^m \mathbb{I}\left(y_i\notin T\left(f_{\boldsymbol{\theta}}, \mathbf{x}_i\right)\right)
\end{align}
\begin{align}
    \mathcal{R}_\mathcal{S}(\mathbb{Q}) = & \mathbb{E}_{(\mathbf{x},y)\sim \mathcal{S}}\mathbb{E}_{\boldsymbol{\theta} \sim \mathbb{Q}}\left[\mathbb{I}\left(y\notin T\left(f_{\boldsymbol{\theta}}, \mathbf{x}\right)\right)\right] \nonumber \\
    = &  \frac{1}{m}\sum_{i=1}^m\mathbb{E}_{\boldsymbol{\theta}\sim \mathbb{Q}} \left[ \mathbb{I}\left(y_i\notin T\left(f_{\boldsymbol{\theta}}, \mathbf{x}_i\right)\right)\right],
\end{align}
where $(\mathbf{x}_i,y_i)\in \mathcal{S}$ and $m=|\mathcal{S}|$.

\subsection{Decision boundary}
Intuitively, if the output $k$-dimensional vector $f_{\boldsymbol{\theta}}(\mathbf{x})$ on the input example $\mathbf{x}$ has a tie, {\it i.e.}, more than one {entry} of the vector have the maximum value, then $\mathbf{x}$ is considered to locate on the decision boundary of $f_{\boldsymbol{\theta}}$. Formally, the decision boundary can be formally defined as below.

\begin{definition}[decision boundary]
\label{def:decision boundary}
Let $f_{\boldsymbol{\theta}}(\mathbf{x}): \mathbb{R}^n \rightarrow \mathbb{R}^k$ be a neural network for classification parameterized by $\boldsymbol{\theta}$, where $n$ and $k$ are the dimensions of input and output, respectively. Then, the \textit{decision boundary} of $f_{\boldsymbol{\theta}}$ is defined by
\begin{equation}
\{\mathbf{x}\in \mathbb{R}^n| \exists i\neq j \in [k], f_{\boldsymbol{\theta}}^{(i)}(\mathbf{x})=f_{\boldsymbol{\theta}}^{(j)}(\mathbf{x}) =\max_q f^{(q)}_{\boldsymbol{\theta}}(\mathbf{x})\}.
\end{equation}
\end{definition}

Consequently, we have the following remark.

\begin{remark}
(1) If an input example $(\mathbf{x},y)$ is not located on the decision boundary of $f_{\boldsymbol{\theta}}$, $T(f_{\boldsymbol{\theta}},\mathbf{x})$ is a singleton, and we have,
\begin{equation}
    \mathbb{I}\left(y\notin T\left(f_{\boldsymbol{\theta}}, \mathbf{x}\right)\right) = \mathbb{I}\left(y \neq T\left(f_{\boldsymbol{\theta}}, \mathbf{x}\right)\right).
\end{equation}

(2) If the input $\mathbf{x}$ is a boundary point, in practice, we randomly draw a label from the set $T(f_{\boldsymbol{\theta}},\mathbf{x})$ as the prediction of $f_{\boldsymbol{\theta}}$ on $\mathbf{x}$. 
\end{remark}

\subsection{Adversarial training}

Adversarial training (AT) 
can enhance the adversarial robustness of neural networks against adversarial examples, which is generated through a popular approach projected gradient descent (PGD) \citep{madry2018towards} in our empirical studies as an example. More specifically, adversarial training can be formulated as solving the minimax-loss problem as follows,
\begin{equation}
\min _{\boldsymbol{\theta}} \frac{1}{m} \sum_{i=1}^{m} \max _{\left\|\mathbf{x}_{i}^{\prime}-\mathbf{x}_{i}\right\| \leq \gamma} \ell\left(f_{\boldsymbol{\theta}}\left(\mathbf{x}_{i}^{\prime}\right), y_{i}\right), 
\end{equation}
where $\gamma$ is the {\it radius} to limit the distance between adversarial examples and original examples. Intuitively, adversarial training enlarges the distances from training examples to decision boundaries to at least $\gamma$, which is formerly very {close} to the decision boundary.

{
\subsection{Model selection}

For a neural network $f_{\boldsymbol{\theta}}$, it is specified by a model $\mathcal{M}$ which consists of a network architecture and hyperparameters. For a list of model candidates $\mathcal{M}$ and given the training data $\mathcal{S}$, {\it log marginal likelihood} (LML) is often employed to select the best model from these candidates \citep{mackay2003information}, as shown as below:
\begin{equation}
    \mathcal{M}^\ast = \arg\max_{\mathcal{M}}\log p(\mathcal{S}|\mathcal{M}),
\end{equation}
and
\begin{equation}
    p(\mathcal{S}|\mathcal{M}) = \int_{\boldsymbol{\theta}} p(\mathcal{S}|\boldsymbol{\theta}, \mathcal{M}) p(\boldsymbol{\theta}|\mathcal{M}) \text{d}\boldsymbol{\theta}.
\end{equation}

Because the parameter $\boldsymbol{\theta}$ is hard to enumerate, the Laplace's method is often used to estimate the LML \citep{mackay1992practical}:
\begin{align}
\log p(\mathcal{S}|\mathcal{M}) & \approx \log q(\mathcal{S}|\mathcal{M}) \\
& \coloneqq \log p(\mathcal{S},\boldsymbol{\theta}^\ast|\mathcal{M}) - \frac{1}{2} \log \left |\frac{1}{2\pi} \mathbb{H}_{\boldsymbol{\theta}^\ast} \right|,
\end{align}
where $q(\mathcal{S}|\mathcal{M})$ is a Gaussian approximation to $p(\mathcal{S}|\mathcal{M})$, $\boldsymbol{\theta}^\ast$ is the optimum to local minimum, and the Hessian $\mathbb{H}_{\boldsymbol{\theta}^\ast} = -\nabla_{\boldsymbol{\theta}^\ast}^2 \log p(\mathcal{S},\boldsymbol{\theta}^\ast|\mathcal{M}) $. In this paper, we follow the implementation in \citep{laplace2021} to calculate the LML of neural networks.

}

\section{Algorithm decision boundary variability}
\label{sec:adbv}
Due to the randomness of learning algorithms, there is no doubt that the parameters have a substantial variation in training repeats. However, {\it do the decision boundaries in the different training repeats have a large discrepancy?} 
Quantitatively, we define the algorithm decision boundary variability (AV) to measure the variability of DBs caused by the randomness of algorithms in different training repeats.

\begin{definition}[algorithm decision boundary variability]
\label{def: consistencey}
Let $f_{\boldsymbol{\theta}}(\mathbf{x}): \mathbb{R}^n \rightarrow \mathbb{R}^k$ be a neural network for classification parameterized by $\boldsymbol{\theta}$. Suppose $\mathbb{Q}(\boldsymbol{\theta})$ is the distribution over $\boldsymbol{\theta}$. Then, the algorithm decision boundary variability for $f_\mathbb{Q}$ on $\mathcal{D}$ is defined as below,
\begin{equation}
    AV(f_\mathbb{Q},\mathcal{D})
    = \mathbb{E}_{(\mathbf{x},y)\sim \mathcal{D}}\mathbb{E}_{\boldsymbol{\theta},\boldsymbol{\theta}^\prime \sim \mathbb{Q}}\left[\mathbb{I}\left(T(f_{\boldsymbol{\theta}}, \mathbf{x}) \neq T(f_{\boldsymbol{\theta}^\prime}, \mathbf{x})\right)\right], 
\end{equation}
where $T(f_{\boldsymbol{\theta}}, \mathbf{x})=\{i\in [k]| f^{(i)}_{\boldsymbol{\theta}}(\mathbf{x})=\max_{j} f^{(j)}_{\boldsymbol{\theta}}(\mathbf{x}) \}$.
\end{definition}

According to the definition, algorithm DB variability reflects the similarity of decision boundaries during different training repeats. An illustration of algorithm DB variability is provided in Figure \ref{figure:illustration algorithm DB}.


\begin{figure*}[t]
\centering
\subfigure[Algorithm DB variability and test error vs. training time]{
\begin{minipage}[b]{\textwidth}
\centering
    		\includegraphics[width=0.235\columnwidth]{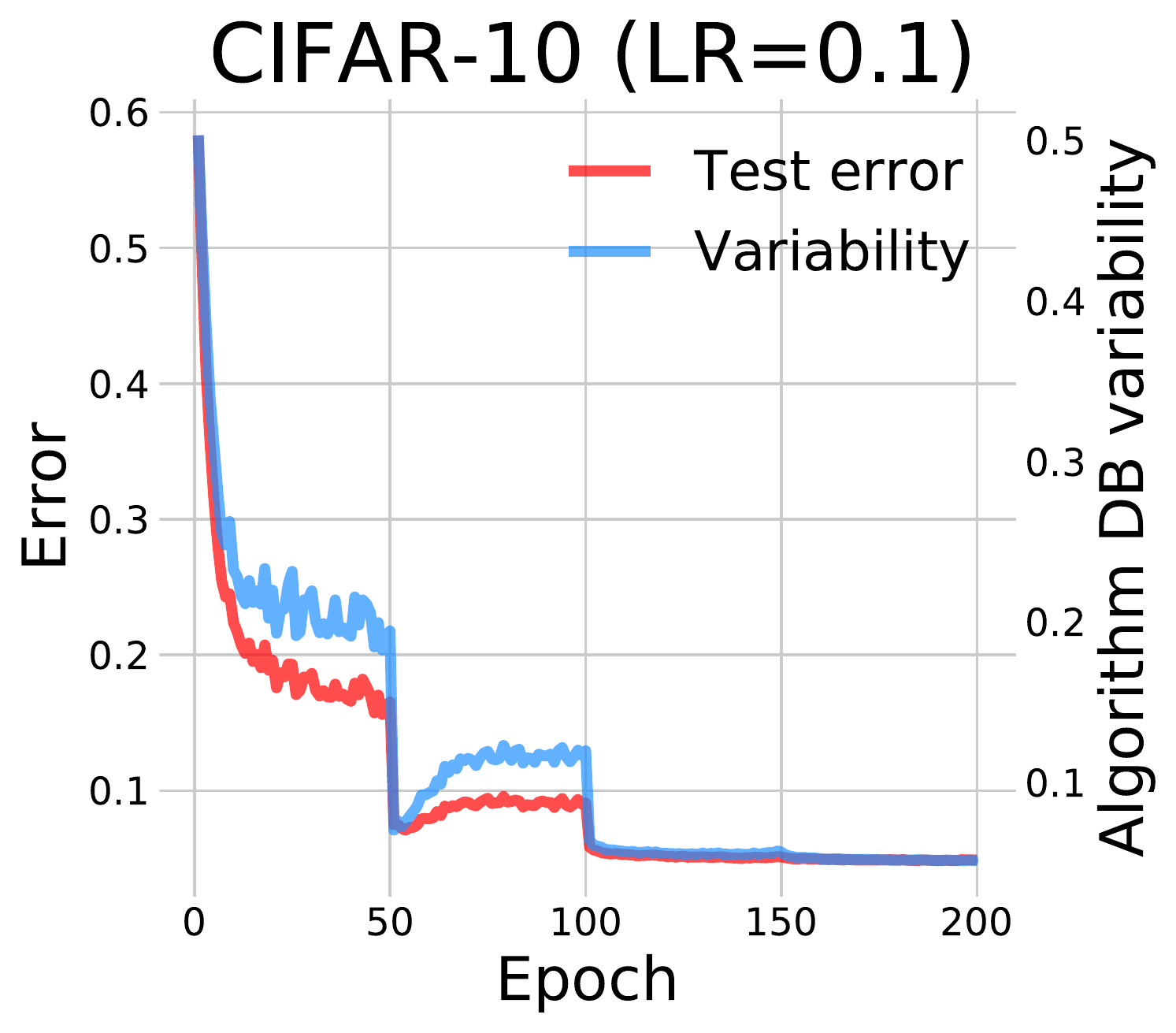}
    		\includegraphics[width=0.24\columnwidth]{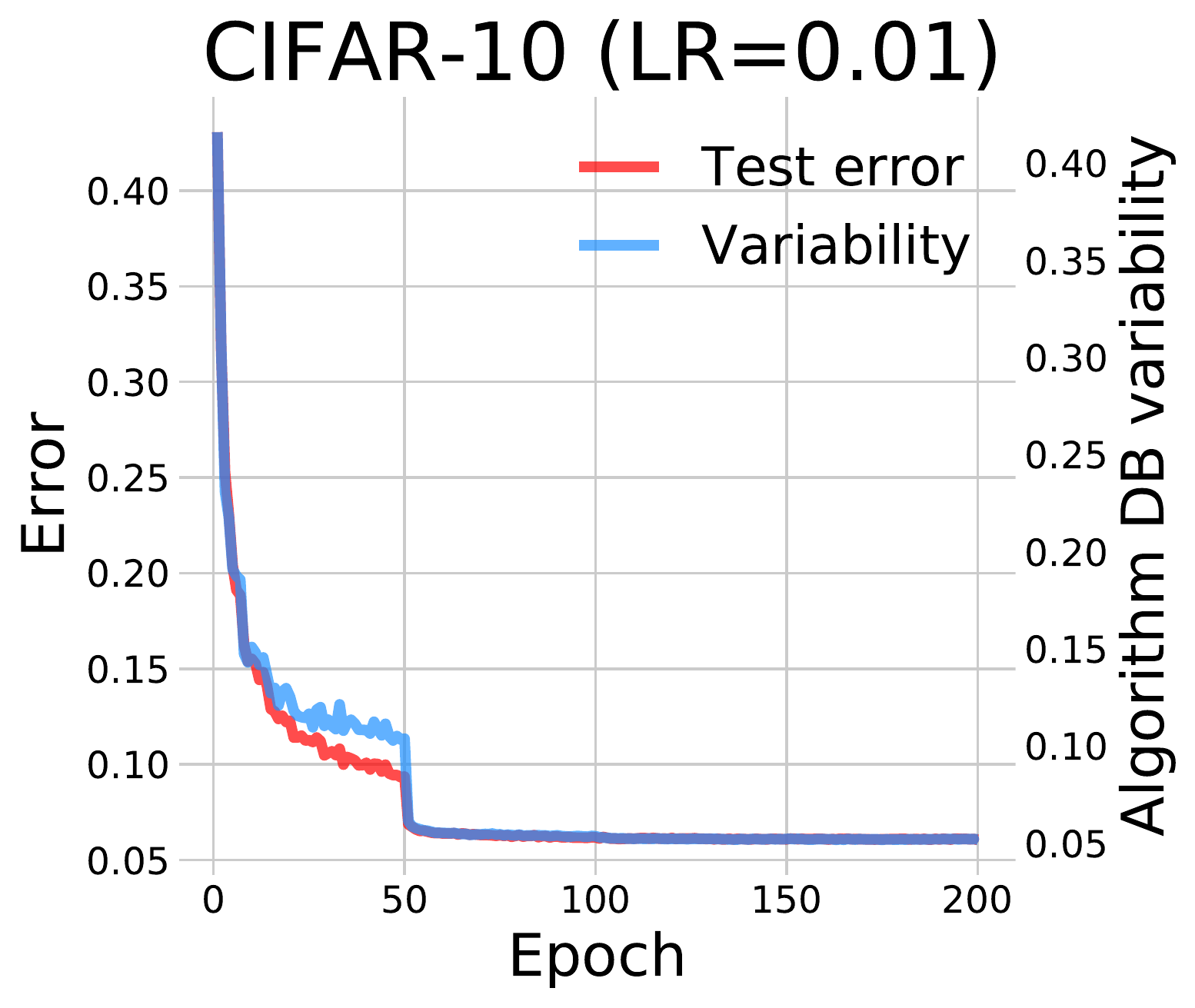}
    		\includegraphics[width=0.23\columnwidth]{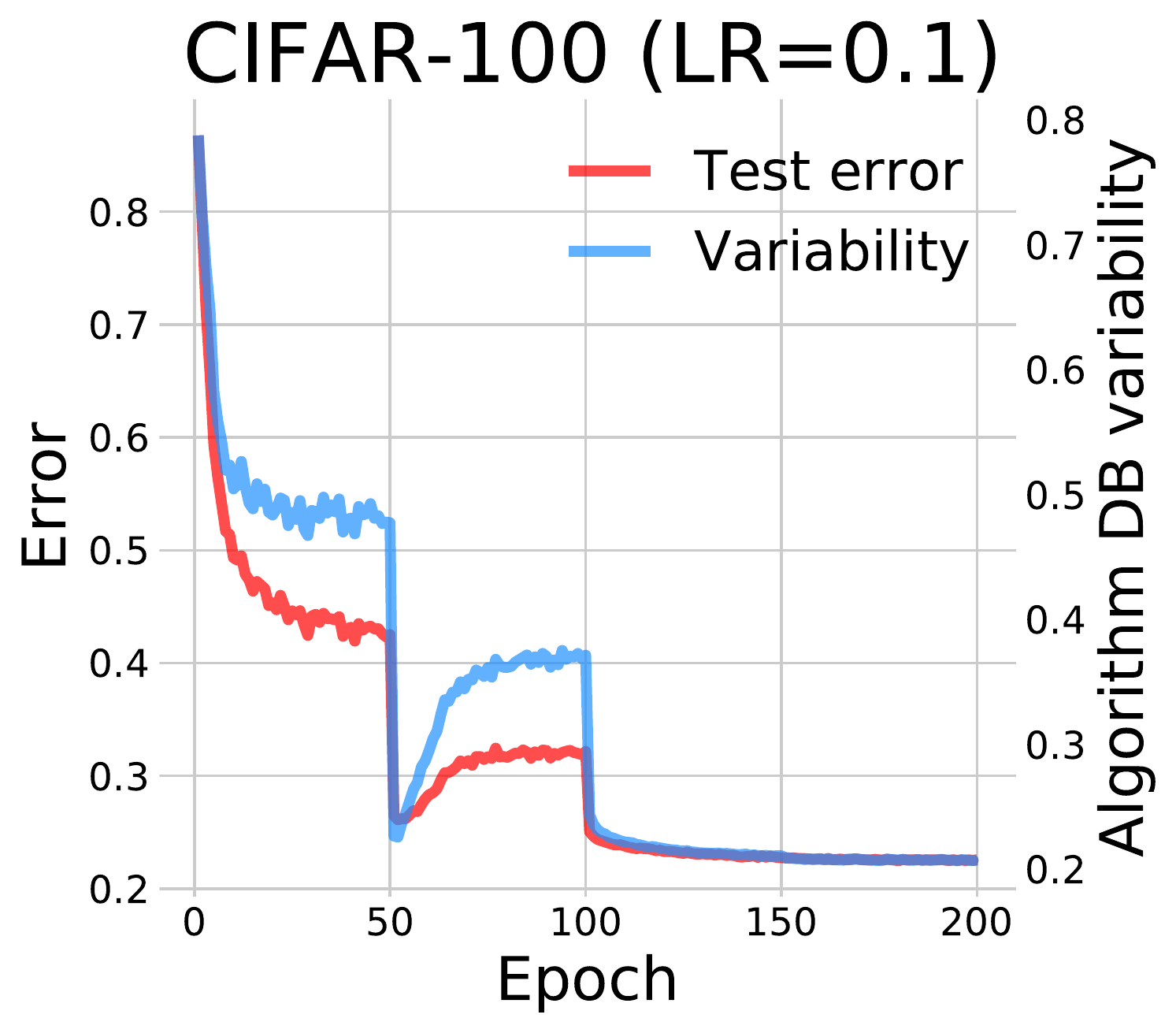}
    		\includegraphics[width=0.23\columnwidth]{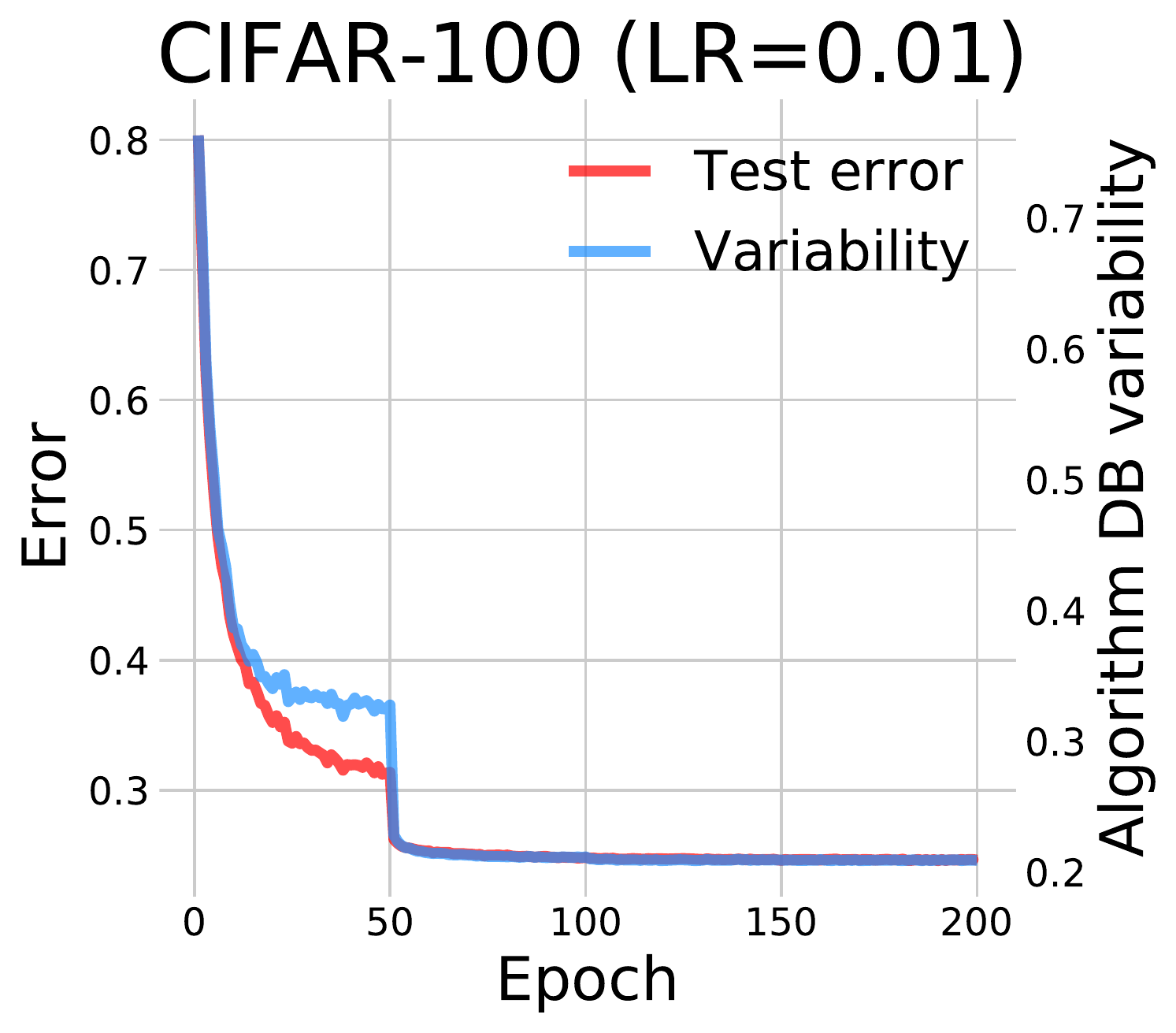}
    		\end{minipage}
		\label{figure:training process plot}   
    	}
\subfigure[Test error vs. algorithm DB variability]{
\begin{minipage}[b]{\textwidth}
\centering
    		\includegraphics[width=0.23\columnwidth]{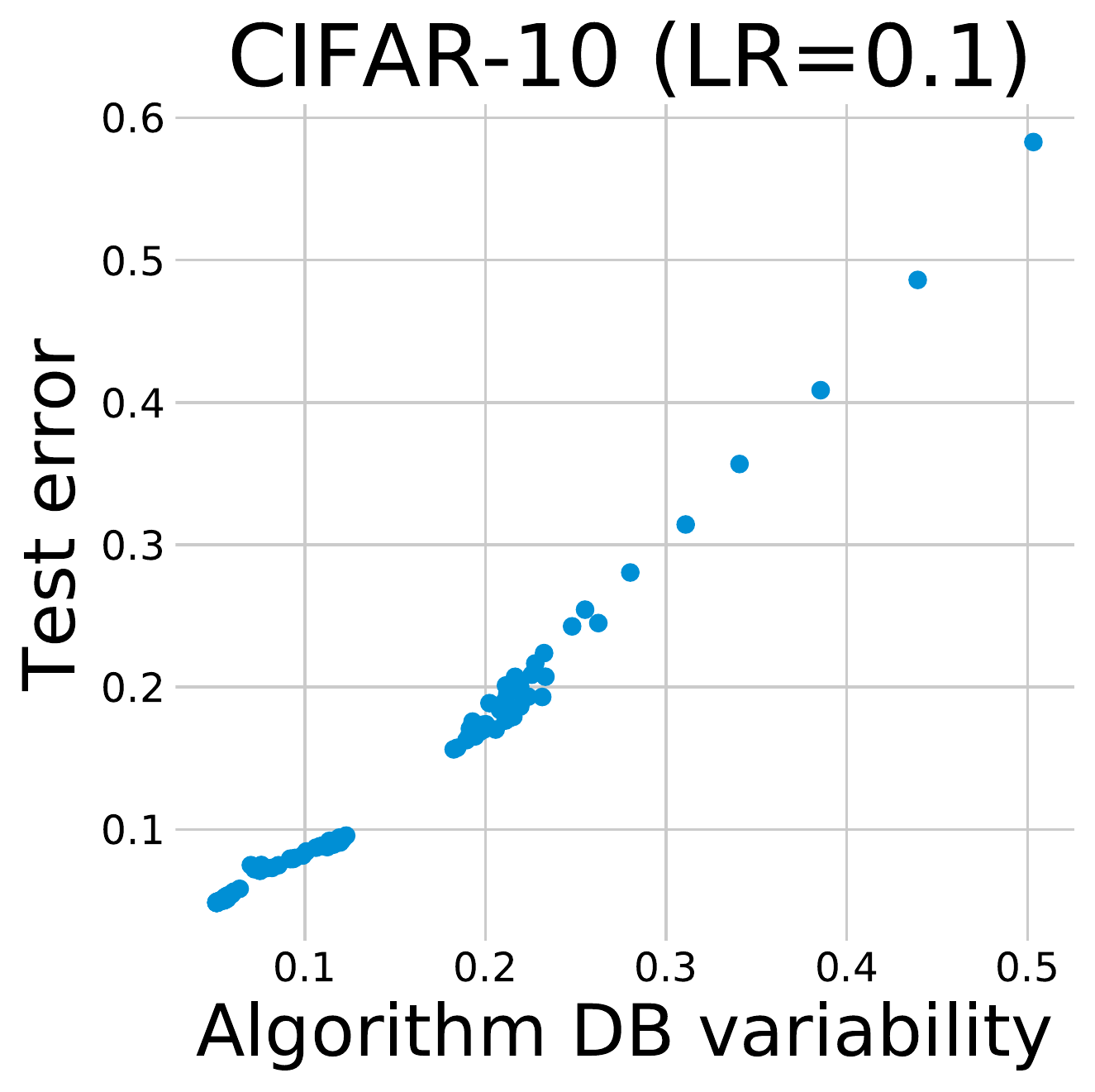}
    		\includegraphics[width=0.233\columnwidth]{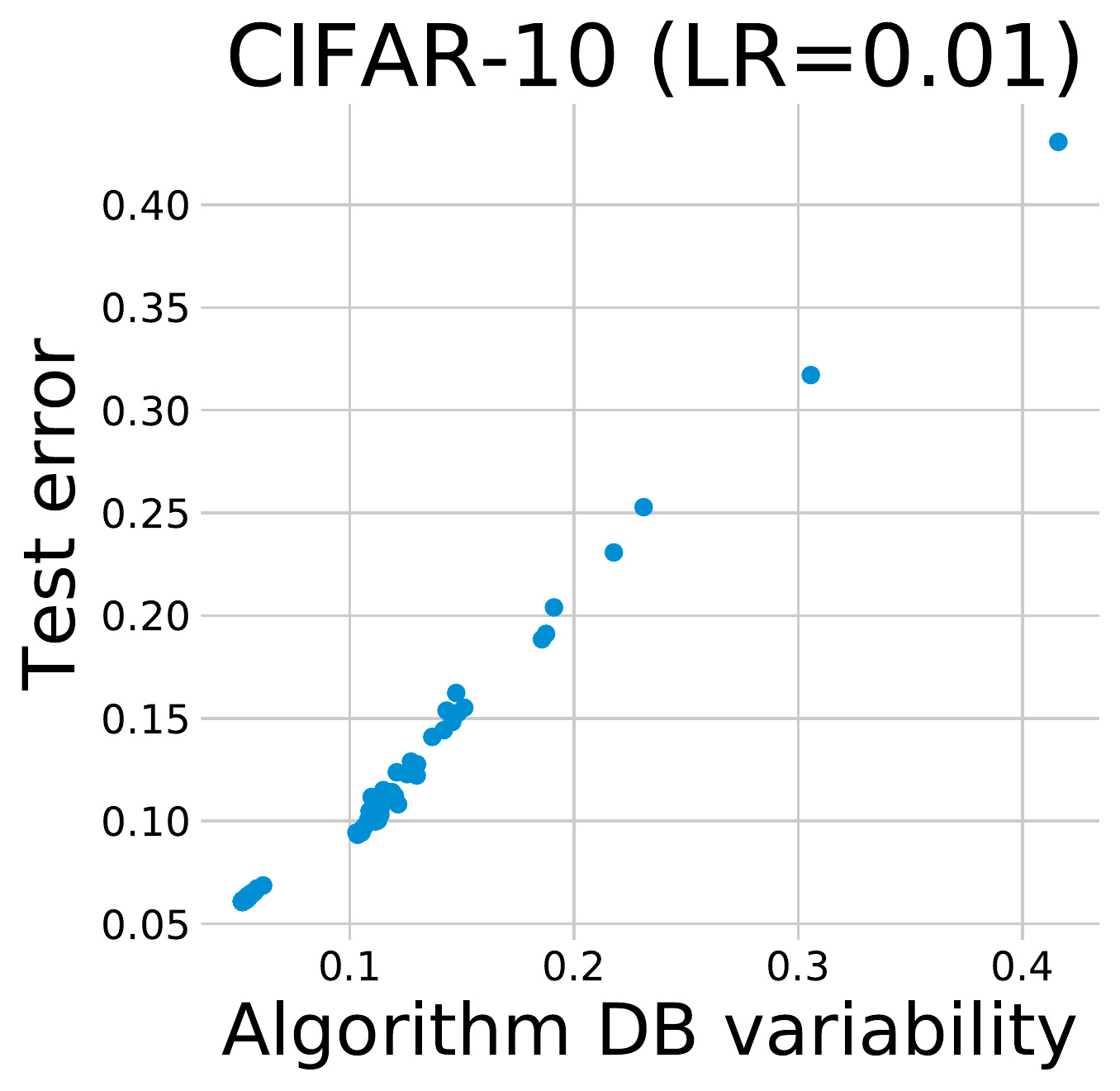}
    		\includegraphics[width=0.23\columnwidth]{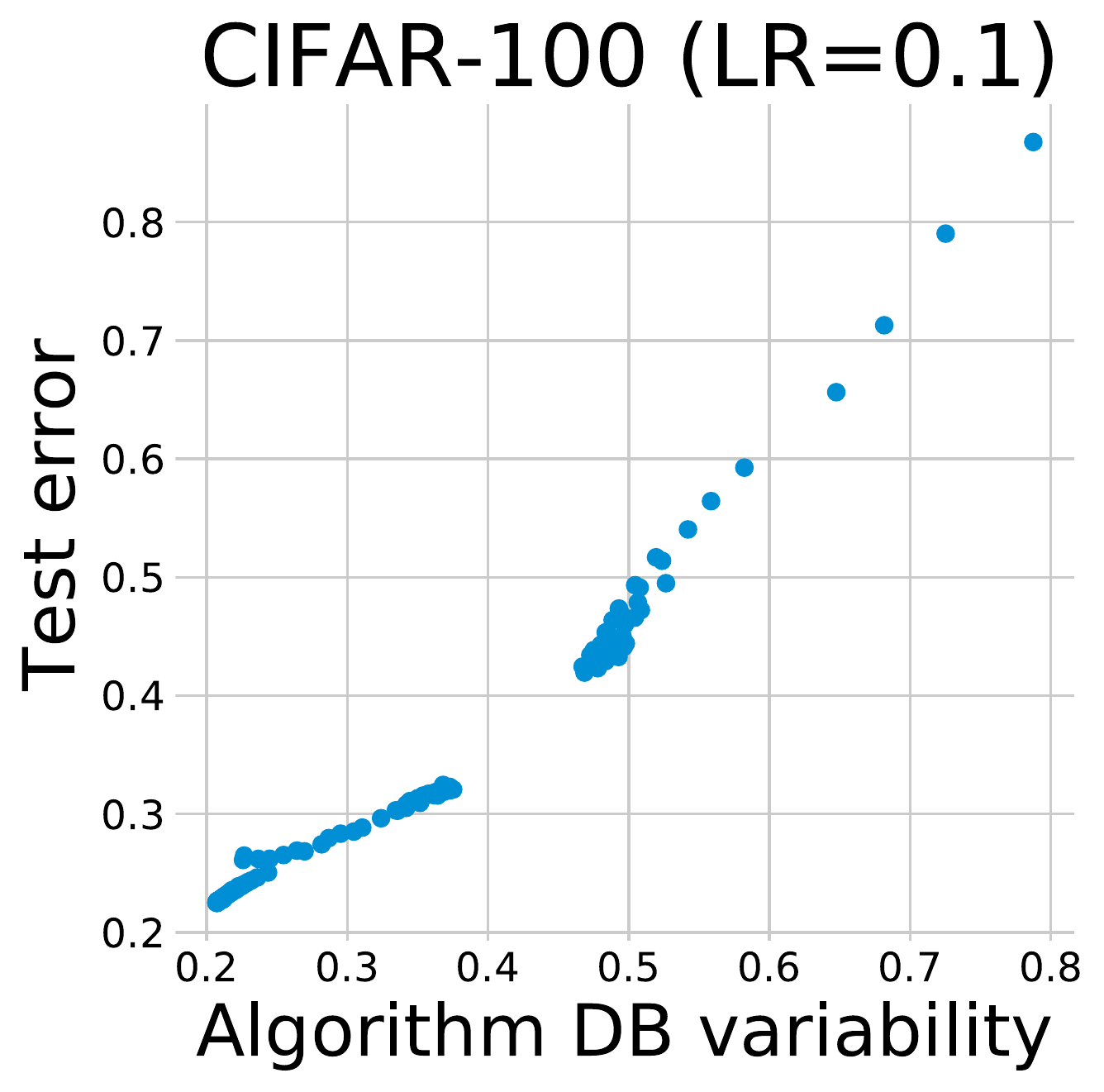}
    		\includegraphics[width=0.23\columnwidth]{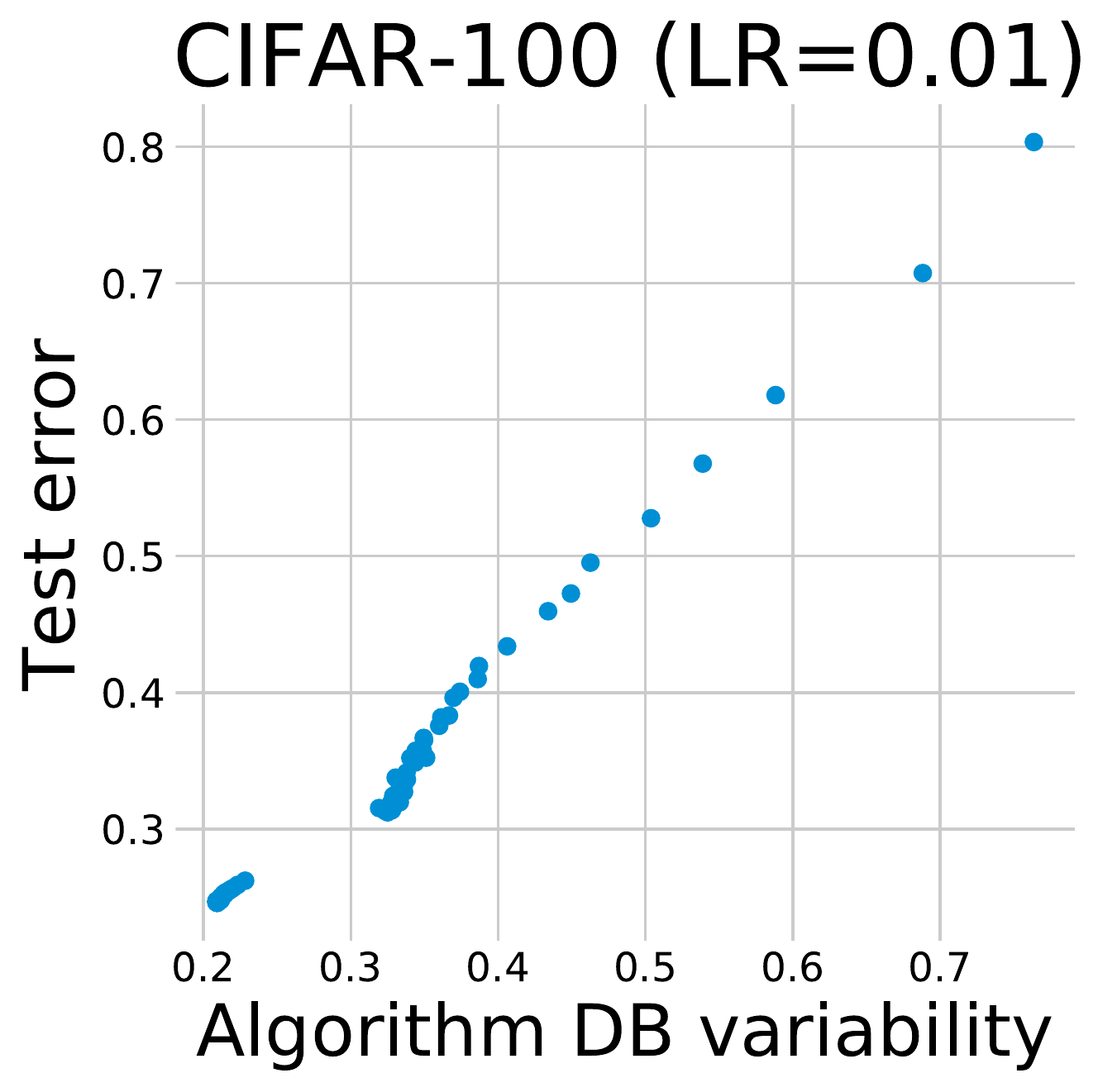}
    		\end{minipage}
		\label{figure:training process scatter plot}   
    	}
\caption{(a) Plots of algorithm DB variability and test error as functions of training time (LR is learning rate). (b) Scatter plots of test error to algorithm DB variability (LR is learning rate). The points are collected from different epochs. Each curve and point is calculated and then averaged on $10$ trials.}
\label{figure:training process}
\end{figure*}

\subsection{Algorithm DB variability and generalization}
\label{sec:real data}

To explore the relationship between the algorithm DB variability and generalization in neural networks, we conduct experiments with various popular network architectures, VGG-16 \citep{simonyan2014very}, ResNet-18 \citep{he2016deep}, and WideResNet-28 \citep{Zagoruyko2016WRN}, on standard datasets, CIFAR-10 and CIFAR-100. In detail, VGG-16 \citep{simonyan2014very}, ResNet-18 \citep{he2016deep}, and WideResNet-28 \citep{Zagoruyko2016WRN} are optimized by standard, non-data-augmentation and adversarial training, respectively, until the training procedure converges. 
Each training setting (determined by different datasets, architectures, and/or training strategies) is repeated for $10$ trials with different random seeds to estimate the parameter distribution $\mathbb{Q}(\boldsymbol{\theta})$. In order to simulate the data generating distribution, we trained two conditional BigGANs \citep{zhao2020differentiable} 
to produce $100,000$ fake (or, synthetic) images for CIFAR-10 and CIFAR-100, respectively. Examples of fake images are shown in Figure \ref{figure:biggan_cifar10} and \ref{figure:biggan_cifar100}.

These generative fake images enables estimating the algorithm DB variability. In every training setting, we plot the average test accuracy vs. the algorithm DB variability; as shown in Figure \ref{figure:result_cifar}. From the plots, we obtain the following four observations: (1) adversarial training dramatically decreases the test accuracy and promotes the algorithm DB variability compared to the standard training. 
(2) data augmentation decreases the algorithm decision boundary variability. 
Intuitively, the images augmented by cropping or flipping are still located on the data generating distribution, so data augmentation can expand the training set. Hence, the expanded training set can characterize wider decision boundaries on the data generating distribution; (3) WideResNet has better test accuracy and lower algorithm DB variability than ResNet and VGG; and (4) a negative correlation exists between the test accuracy and the algorithm DB variability. Based on these observations, we propose the following conjecture.

\begin{hypothesis}
\label{hypothesis:consistency}
{\it Neural networks with smaller algorithm decision boundary variability on the data generating distributions possess better generalization performance.}
\end{hypothesis}

We then conduct experiments on the algorithm DB variability {\it w.r.t.} training time, sample size, and label noise {to concrete this hypothesis}.

\begin{figure*}[t]
\centering
\subfigure[Algorithm DB variability vs. sample size]{
\begin{minipage}[b]{0.46\textwidth}
\centering
    		\includegraphics[width=0.485\columnwidth]{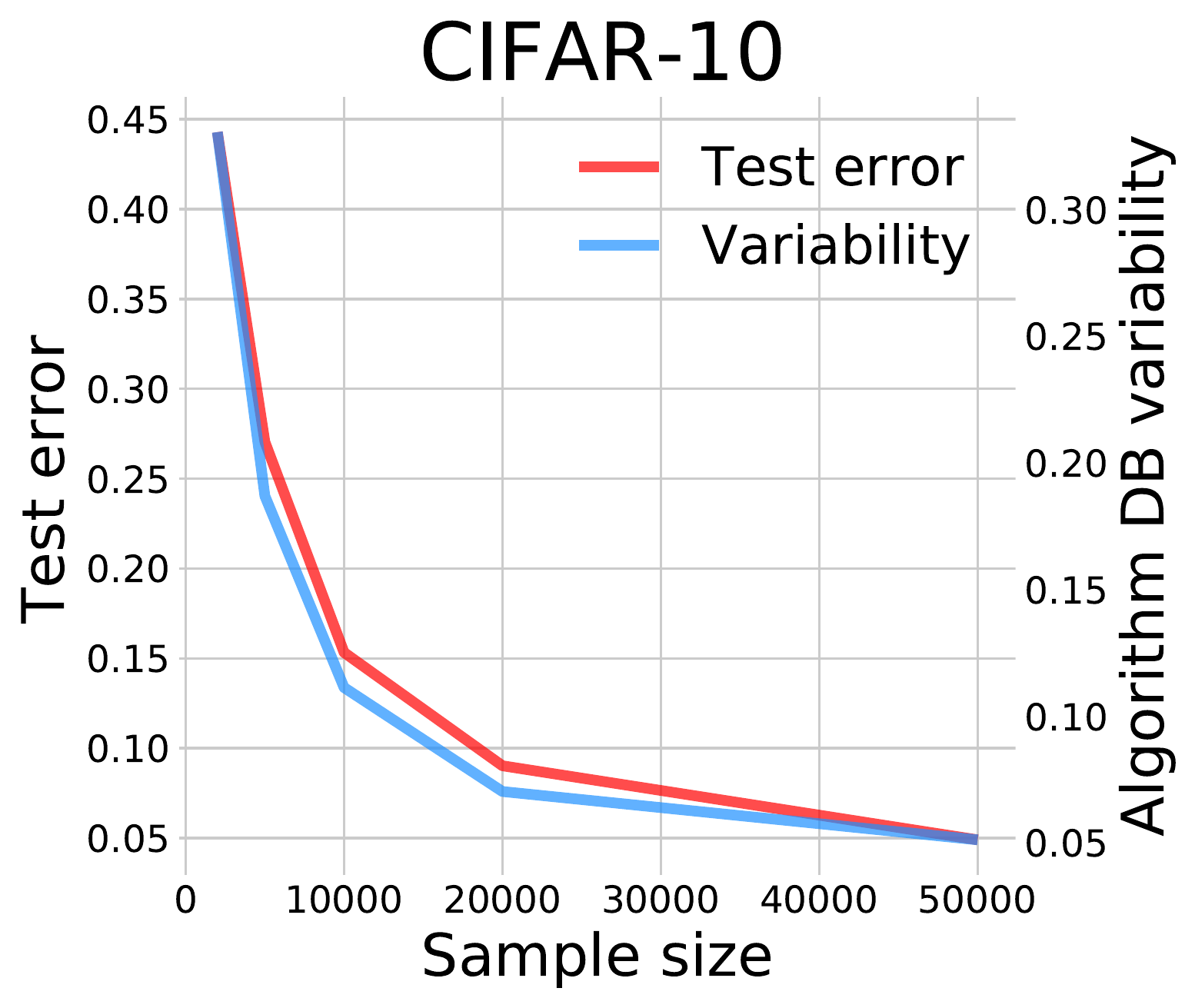}
    		\includegraphics[width=0.48\columnwidth]{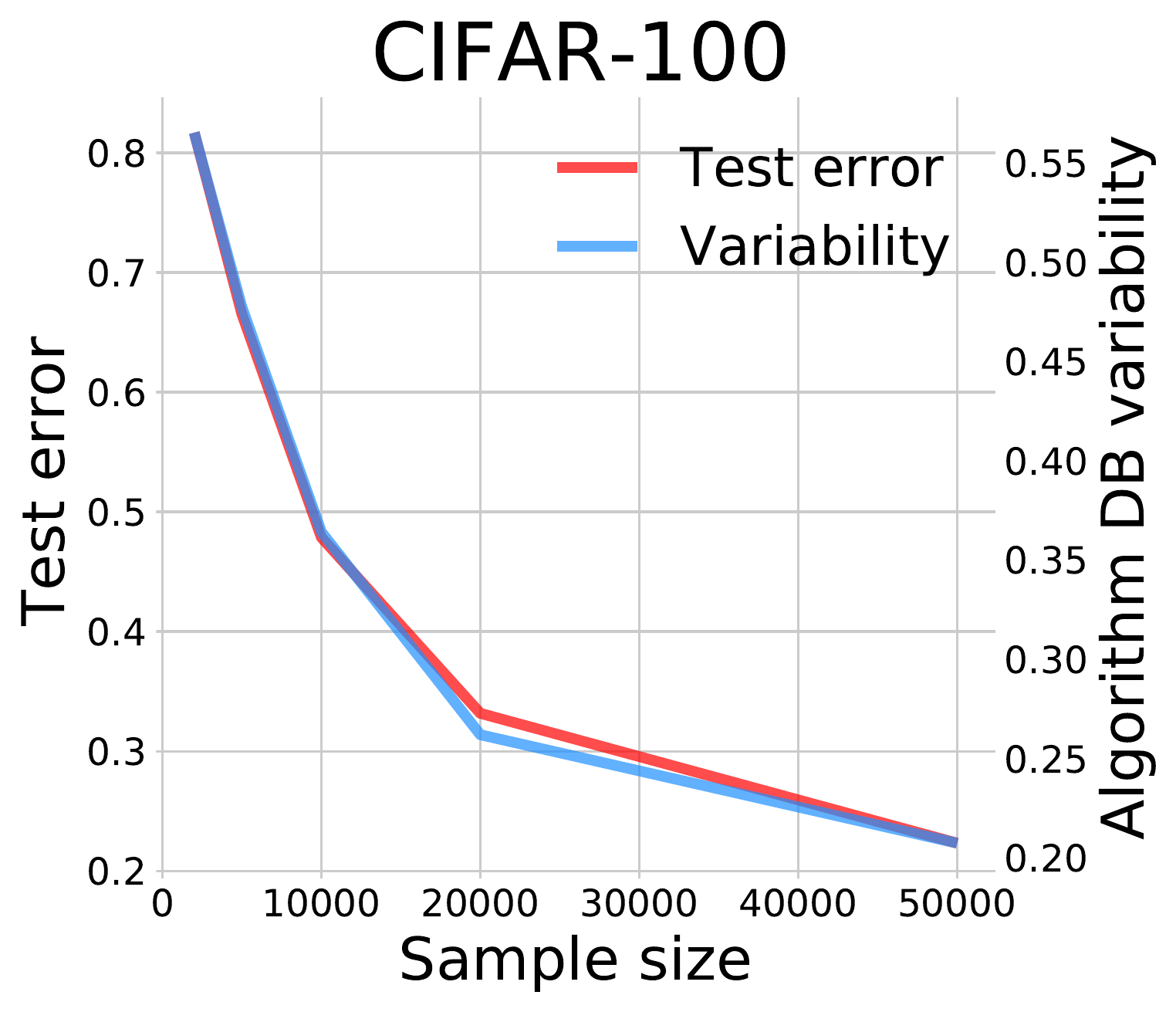}
    		\end{minipage}
		\label{figure:sample size plot}   
    	}
\subfigure[Algorithm DB variability vs. time (label noise)]{
\begin{minipage}[b]{0.46\textwidth}
\centering
    		\includegraphics[width=0.485\columnwidth]{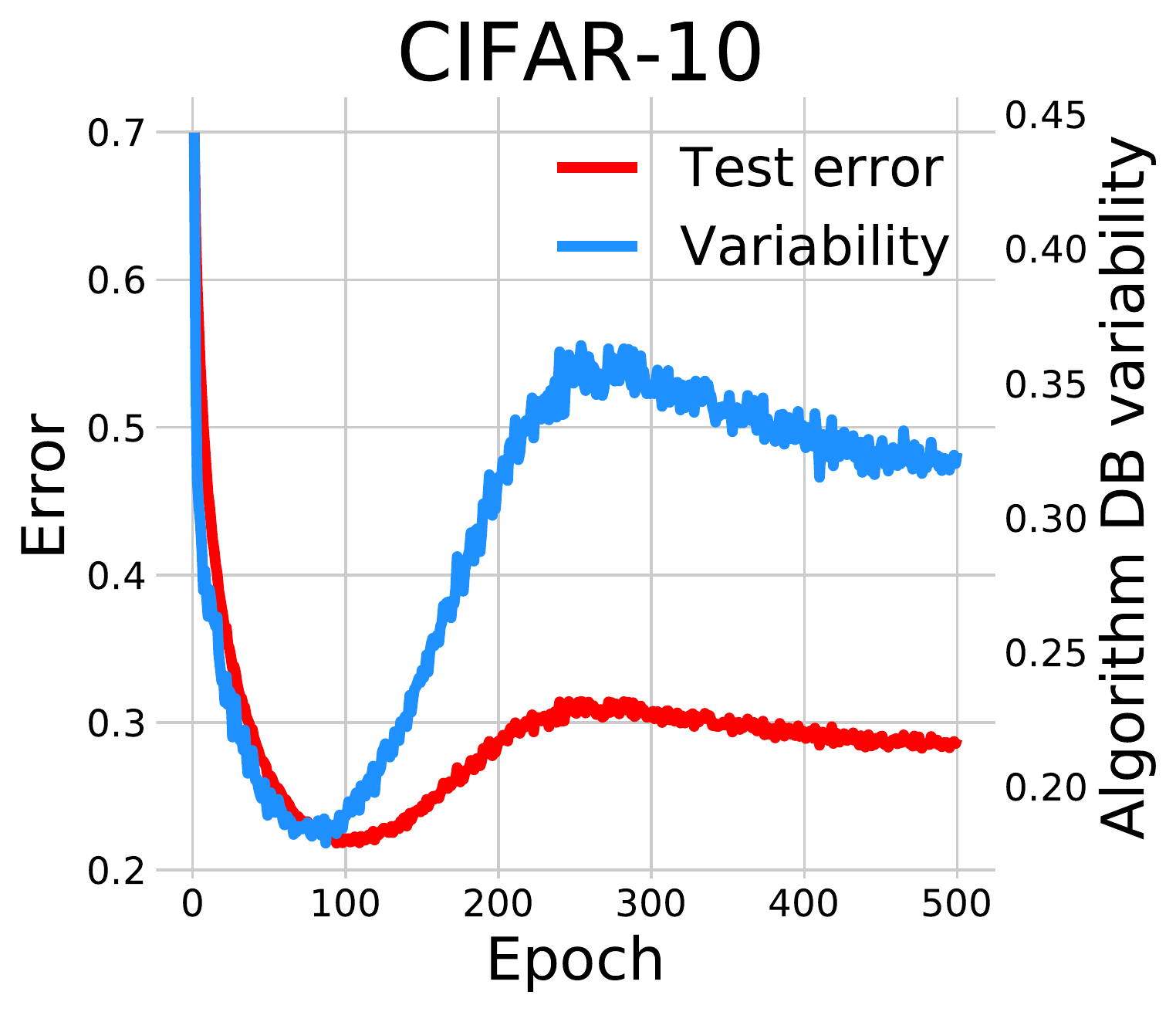}
    		\includegraphics[width=0.48\columnwidth]{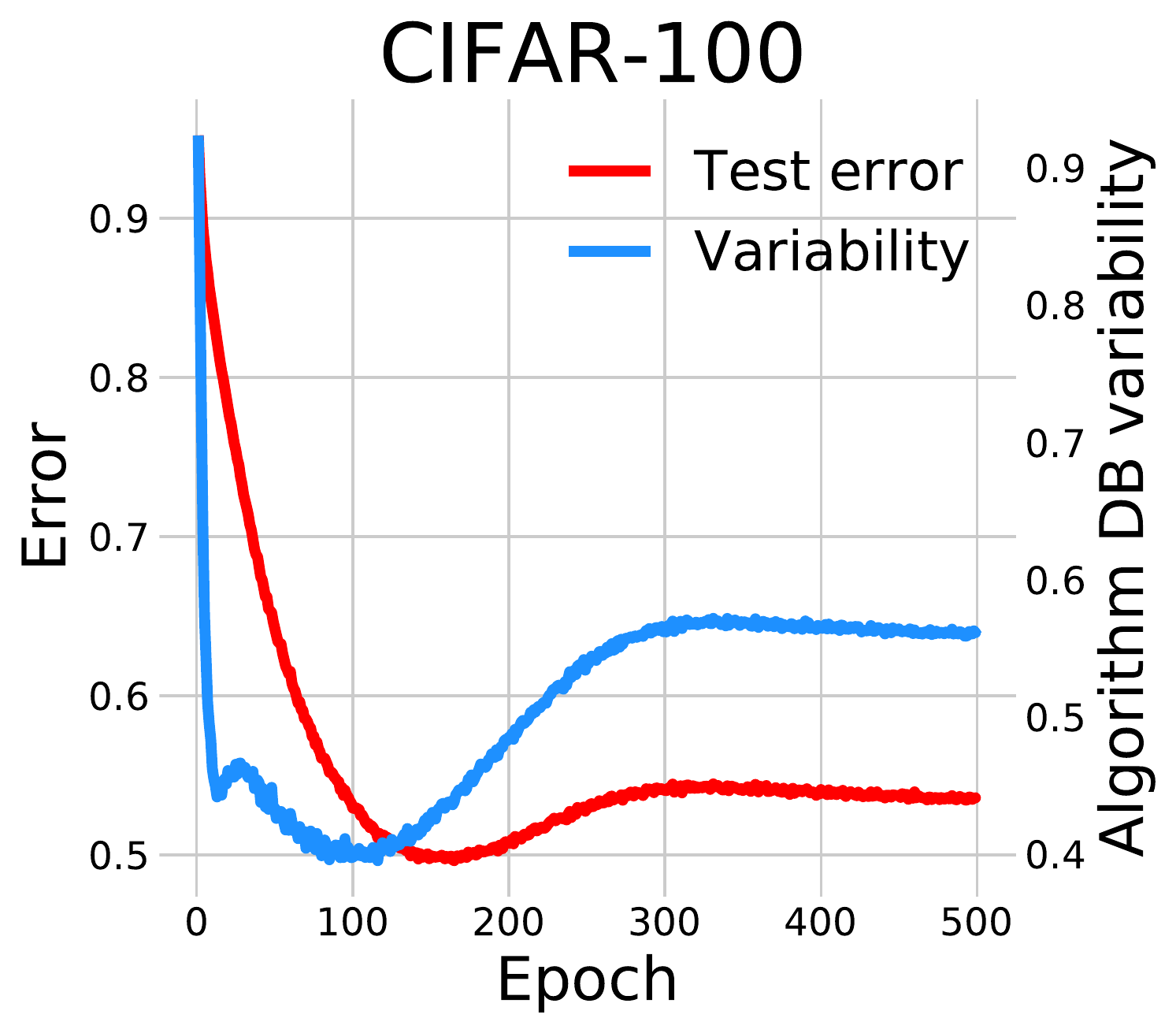}
    		\end{minipage}
		\label{figure:label noise}
    	}
\caption{(a) Plots of algorithm DB variability and test error as functions of training sample size on CIFAR-10 and CIFAR-100. (b) Plots of algorithm DB variability and test error as functions of training time with the existence of $20\%$ label noise on CIFAR-10 and CIFAR-100. Each curve is calculated and then averaged on $10$ trials.}
\label{figure:sample size}
\end{figure*}

\subsection{Algorithm DB variability and training time}
\label{sec:training process}
To investigate the relationship between algorithm DB variability and training time, we train $40$ ResNet-18 with different initial learning rates of $0.1$ and $0.01$ on CIFAR-10 and CIFAR-100. Then, the algorithm DB variability and test error are calculated at each epoch; see Figure \ref{figure:training process plot}. From the plots, two observations can be obtained: (1) algorithm DB variability and test error share a very similar curve {\it w.r.t.} the training time; and (2) algorithm DB variability decreases during the training process. The decline of algorithm DB variability shows that the interpolation on examples reduces the variability of decision boundaries on data generating distribution. As shown in Figure \ref{figure:training process scatter plot}, we collect the points of (algorithm DB variability, test error) from different epochs, and the scatter plots present a significant positive correlation between test error and the algorithm DB variability, and thus supports Hypothesis \ref{hypothesis:consistency}.

\subsection{Algorithm DB variability and sample size}
\label{sec:sample size}
We next investigate how sample size influences the algorithm DB variability. $100$ ResNet-18 are trained on five training sample sets of different sizes randomly drawn from CIFAR-10 and CIFAR-100, while all irrelevant variables are strictly controlled. Then, the algorithm DB variability and test error are calculated in all cases; see Figure \ref{figure:sample size plot}. From the plots, we have the following two observations: (1) test error and algorithm DB variability share a very similar curve {\it w.r.t.} the training sample size; (2) larger sample size, which intuitively helps obtain a more smooth estimation of the decision boundary, also contributes to smaller algorithm DB variability; and (3) there is a significant positive correlation between test error and algorithm DB variability, which fully supports our argument of Hypothesis \ref{hypothesis:consistency}.

\subsection{Algorithm DB variability and label noise}
\label{sec:label noise}
\citet{belkin2019reconciling,nakkiran2019deep} show a surprising epoch-wise double descent of test error, especially with the existence of label noise. We explore in this section the trend of algorithm DB variability when the label noise exists. $20$ ResNet-10 are trained for $500$ epochs with a constant learning rate of $0.0001$ on CIFAR-10 and CIFAR-100 with $20\%$ label noise. We clarify that the noise labels remain constant in different training repeats, which is necessary to estimate the algorithm DB variability. Then, the average test error and algorithm DB variability are calculated at every training epoch, as shown in Figure \ref{figure:label noise}. From the plots, two observations can be derived: (1) the algorithm DB variability also undergoes an epoch-wise double descent during the training process, especially in the left panel of Figure \ref{figure:label noise}; and (2) test error and algorithm DB variability still share a very similar curve {\it w.r.t.} the training time with the existence of label noise, which implies that factors influence the generalization of networks can also have an influence on the algorithm DB variability. Hence, algorithm DB variability is an excellent indicator for the generalization ability of networks. 

Here, we propose an insightful explanation about the epoch-wise double descent phenomenon of the data DB variability {\it w.r.t.} the training time: the increase of algorithm DB variability shows that fitting the noisy examples has a considerable effect on the formation of decision boundaries on data generating distribution. However, the algorithm DB variability decreases when the training proceeds further. This indicates that the negative effects brought from fitting the noisy training examples is gradually weakened. 
In other words, as the training proceeds, neural networks can automatically reduce the impact brought from interpolating hard-to-fit examples, 
which is insightful in explaining the decent generalizability of neural networks. 

{
\subsection{DB variability and model selection}
\label{sec:model selection}

In this section, we investigate the algorithm DB variability in model selection. We employ $25$ CNN and $25$ ResNet with different widths and depths to consist the model candidate pool. Each network is trained on CIFAR-10/100 with the initial learning rate of $0.01$ for $50$ epochs and $0.02$ for the next $50$ epochs. We repeat the training process $5$ times to compute the algorithm DB variability. Then, the algorithm DB variability and average test accuracy are calculated for each model; see Figure \ref{figure:model selection}, and we also plot the correlation between average LML and test accuracy for comparison. For each plot, we also calculate Spearman's rank-order correlation coefficients (SCCs) and the corresponding $p$ value of the collected data to investigate the statistical significance of the correlations; please see the bottom of the plots. For the plots two observation can be obtained: (1) there is a positive correlation between test accuracy and algorithm DB variability ($SCCs>0.9$) and the correlation are statistically significant ($p<0.005$)\footnote{The definition of ``statistically significant'' has various versions, such as $p<0.05$ and $p<0.1$. This paper uses a more rigorous one ($p<0.005$).}; (2) compared to the correlation between LML and test accuracy, the correlation between algorithm DB variability and test accuracy is more significant. Therefore, algorithm DB variability is a better measurement for model selection than LML, while both of them do not require a test set for validation.

\begin{figure*}[t]
\centering
\subfigure[CIFAR-10]{
\begin{minipage}[b]{0.46\textwidth}
            \centering
    		\includegraphics[width=0.45\columnwidth]{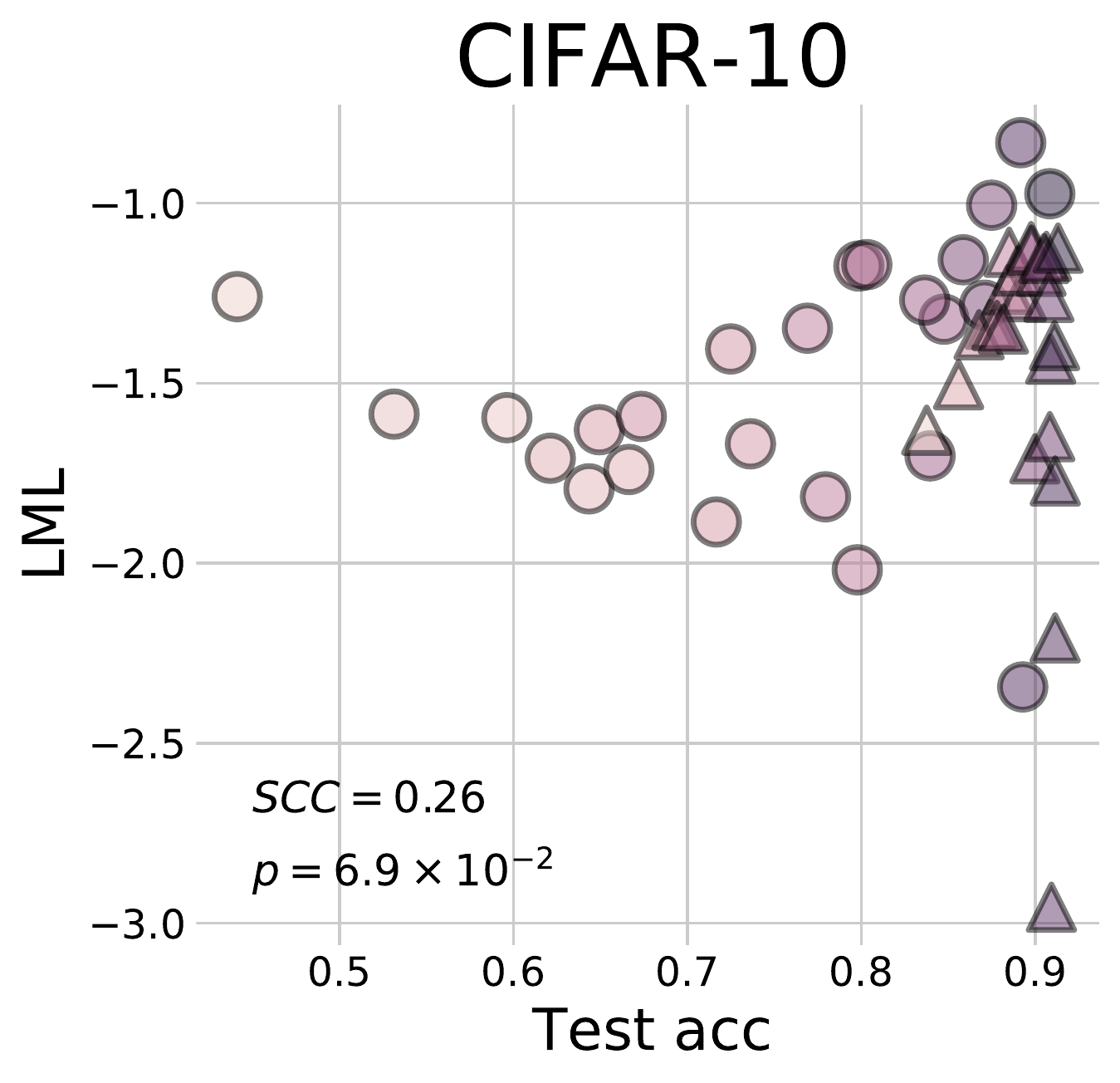}
    		\includegraphics[width=0.45\columnwidth]{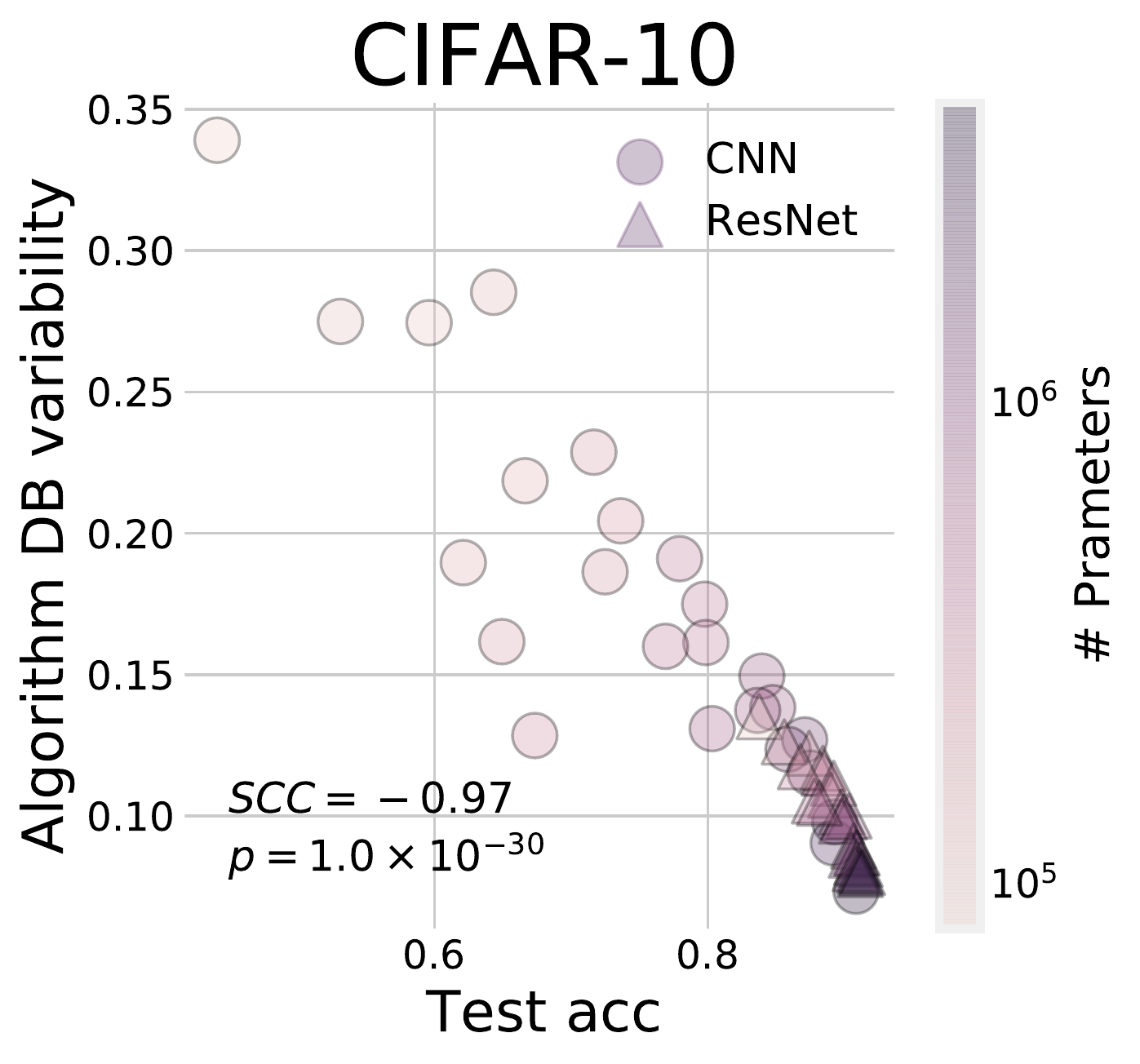}
    		\end{minipage}
		\label{figure:model selection cifar10}
    	}
\subfigure[CIFAR-100]{
\begin{minipage}[b]{0.46\textwidth}
            \centering
    		\includegraphics[width=0.45\columnwidth]{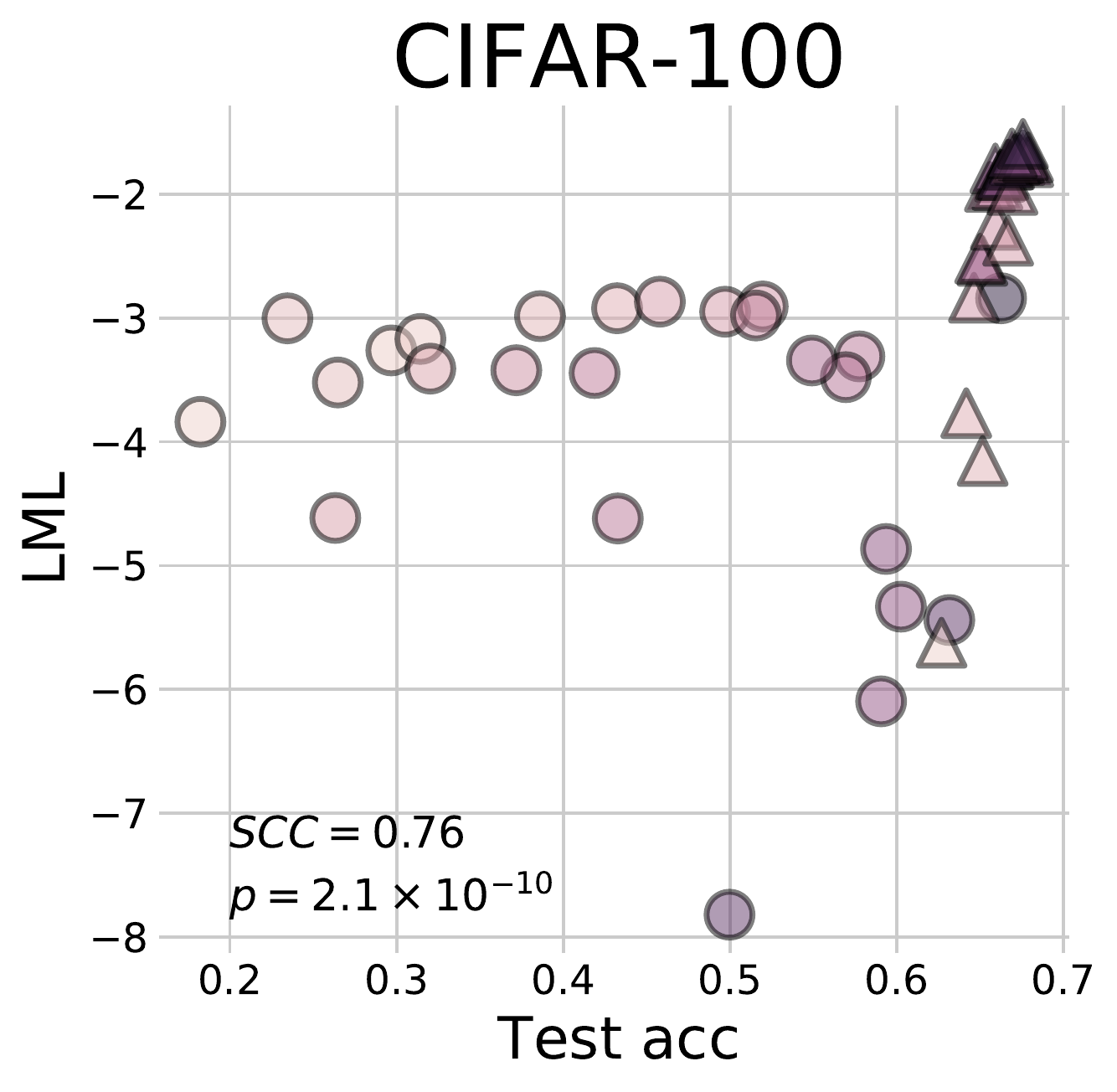}
    		\includegraphics[width=0.45\columnwidth]{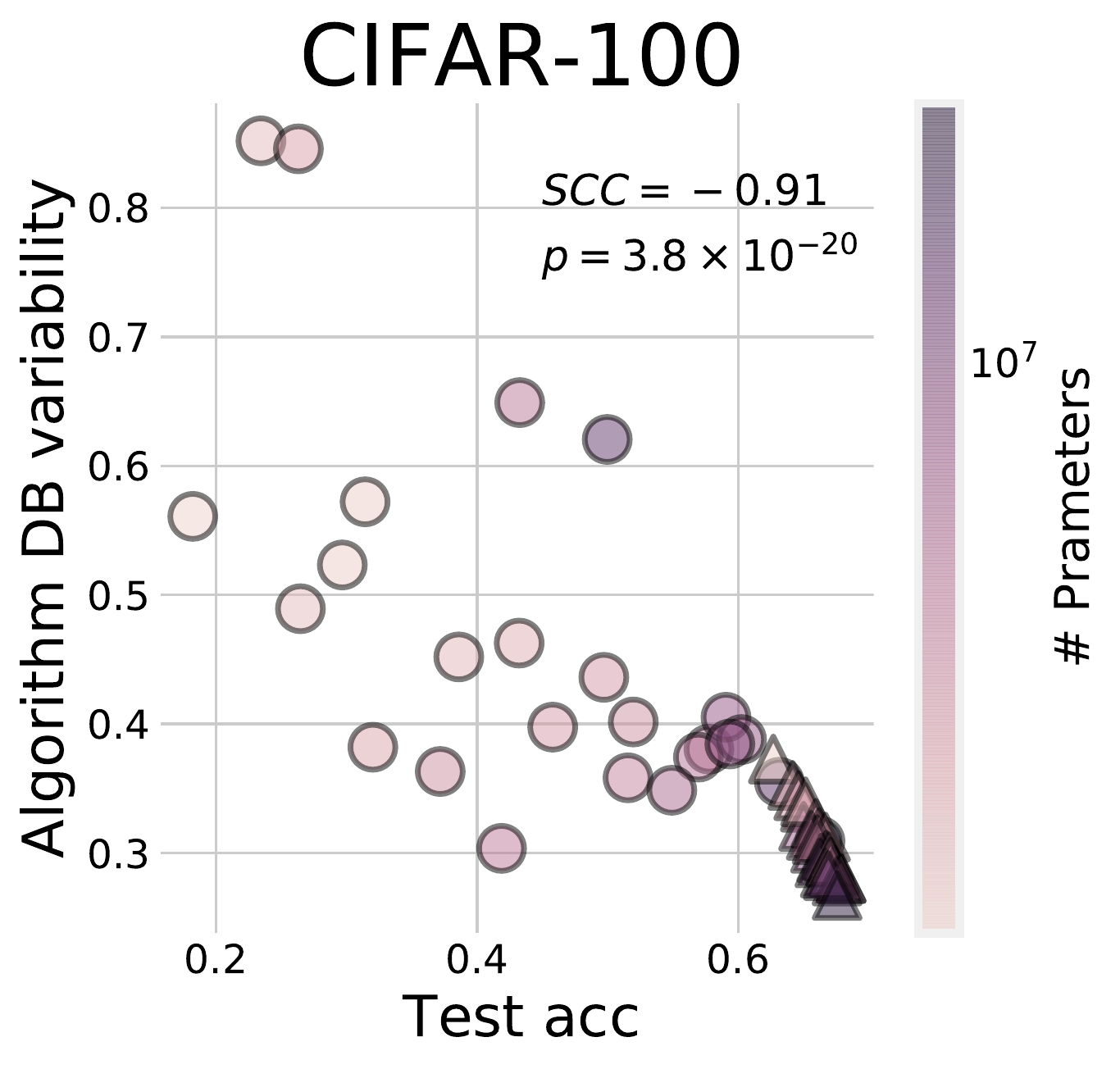}
    		\end{minipage}
		\label{figure:model selection cifar100}
    	}
\caption{Scatter plots of LML and algorithm DB variability to test accuracy with different architectures on CIFAR-10 and CIFAR-100.}
\label{figure:model selection}      
\end{figure*}

}

\subsection{Theoretical Evidence}
\label{sec:theoretical evidence about algorithm variability}

In this section, we explore and develop the theoretical foundations for the algorithm decision boundary variability on data generating distributions. 

In most practical cases, the dimension of decision boundaries is smaller than the data space. For example, the decision boundary in a three-dimensional data space is usually two. Thus, we may make the following mild assumption.

\begin{assumption}
\label{assumption:null boundary}
The decision boundary of the classifier network $f_{\boldsymbol{\theta}}$ on data generating distribution $\mathcal{D}$ is a set with measure zero.
\end{assumption}

We then have the following lemma.

\begin{lemma}
Let $f_{\boldsymbol{\theta}}(\mathbf{x}): \mathbb{R}^n \rightarrow \mathbb{R}^k$ be a classifier network parameterized by $\boldsymbol{\theta}$. If Assumption \ref{assumption:null boundary} holds for all $\boldsymbol{\theta}\sim\mathbb{Q}$, then, for all $i\in \{1,\cdots,k\}$, we have
\begin{equation}
    \mathbb{E}_{(\mathbf{x},y)\sim \mathcal{D}}\left[\mathbb{I}(i\in T(f_{\boldsymbol{\theta}},\mathbf{x}))\right] = \mathbb{E}_{(\mathbf{x},y)\sim \mathcal{D}}\left[\mathbb{I}(T(f_{\boldsymbol{\theta}},\mathbf{x})=i)\right] 
\end{equation}
and
\begin{equation}
    \mathbb{E}_{(\mathbf{x},y)\sim \mathcal{D}}\left[\mathbb{I}(i\notin T(f_{\boldsymbol{\theta}},\mathbf{x}))\right] = \mathbb{E}_{(\mathbf{x},y)\sim \mathcal{D}}\left[\mathbb{I}(T(f_{\boldsymbol{\theta}},\mathbf{x})\neq i)\right]. 
\end{equation}
\end{lemma}

Then, can we prove the following theorem.

\begin{theorem}[lower bound on expected risk]
\label{thm:lower bound}
Let $f_{\boldsymbol{\theta}}(\mathbf{x}): \mathbb{R}^n \rightarrow \mathbb{R}^k$ be a neural network for classification parameterized by $\boldsymbol{\theta}$. Suppose $\mathbb{Q}(\boldsymbol{\theta})$ is the distribution over $\boldsymbol{\theta}$. Then, if Assumption \ref{assumption:null boundary} holds for all $\boldsymbol{\theta}\sim\mathbb{Q}$, we have,
\blue{
\begin{equation}
    \mathcal{R}_\mathcal{D}(\mathbb{Q}) \geq 1- \sqrt{1-AV(f_\mathbb{Q},\mathcal{D})}, 
\end{equation}
}
where $AV(f_\mathbb{Q},\mathcal{D})$ is the algorithm DB variability for $f_\mathbb{Q}$ on data generating distribution $\mathcal{D}$.
\end{theorem}

Theorem \ref{thm:lower bound} provides a lower bound on the expected risk $\mathcal{R}_\mathcal{D}(\mathbb{Q})$ based on the algorithm DB variability $AV(f_\mathbb{Q},\mathcal{D})$. Moreover, when we consider the binary classification, {\it i.e.}, $k=2$, there is a tighter lower bound.

\begin{theorem}[lower bound for binary case]
\label{thm:lower bound for binary case}
Let $f_{\boldsymbol{\theta}}(\mathbf{x}): \mathbb{R}^n \rightarrow \mathbb{R}^2$ be a binary classifier network parameterized by $\boldsymbol{\theta}$ and let $\mathbb{Q}(\boldsymbol{\theta})$ be the distribution over $\boldsymbol{\theta}$. Suppose the expected risk $\mathcal{R}_\mathcal{D}(\mathbb{Q})\leq \frac{1}{2}$ and Assumption \ref{assumption:null boundary} hold for all $\boldsymbol{\theta}\sim \mathbb{Q}$, then we have
\begin{equation}
    \mathcal{R}_{\mathcal{D}}(\mathbb{Q}) \geq \frac{1-\sqrt{1-2AV(f_\mathbb{Q},\mathcal{D})}}{2}. 
\end{equation}
\end{theorem}

\begin{remark}
    \blue{The combination of the empirical and theoretical results suggest a meaningful correlation between algorithm DB variability and expected risk in deep learning.}
\end{remark}


\section{Data decision boundary variability}

In Section \ref{sec:adbv}, we introduced the algorithm DB variability, which measures the decision boundary variability caused by the randomness of learning algorithms. 
In this section, we define the data DB variability to characterize decision boundary variability caused by the randomness in training data.

\begin{definition}[data decision boundary variability]
\label{def:complexity of db}
Let $f_{\boldsymbol{\theta}}(\mathbf{x}): \mathbb{R}^n \rightarrow \mathbb{R}^k$ be a neural network for classification parameterized by $\boldsymbol{\theta}$, where $\boldsymbol{\theta}\sim\mathcal{A}(\mathcal{S})$ is returned by leveraging the stochastic learning algorithm $\mathcal{A}$ on the training set $\mathcal{S}$, which is sampled from the data generating distribution $\mathcal{D}$.
We term $\mathcal{S}_\eta \subset \mathcal{S}$ as a $\eta\textit{-subset}$ of $\mathcal{S}$ if $\frac{|\mathcal{S}_\eta|}{|\mathcal{S}|}=\eta$. Then, if we fixed $\eta$ and
\begin{equation}
    \inf_{\mathcal{S}_\eta \subset \mathcal{S}} \mathbb{E}_{\mathcal{D}} \mathbb{E}_{\boldsymbol{\theta}\sim \mathcal{A}(\mathcal{S}), \boldsymbol{\theta}^\prime \sim \mathcal{A}(\mathcal{S}_\eta)} \left[\mathbb{I}\left(T(f_{\boldsymbol{\theta}}, \mathbf{x}) \neq T(f_{\boldsymbol{\theta}^\prime}, \mathbf{x})\right)\right] = \epsilon, 
\end{equation}
the decision boundary of $f_{\mathcal{A}(\mathcal{S})}$ is said to possess a $(\epsilon, \eta)\textit{-data decision boundary variability}$.
\end{definition}

\begin{figure*}[t]
\centering
\subfigure[$\eta$-$\epsilon$ curves on CIFAR-10]{
\begin{minipage}[b]{0.3\textwidth}
    		\includegraphics[width=1.\columnwidth]{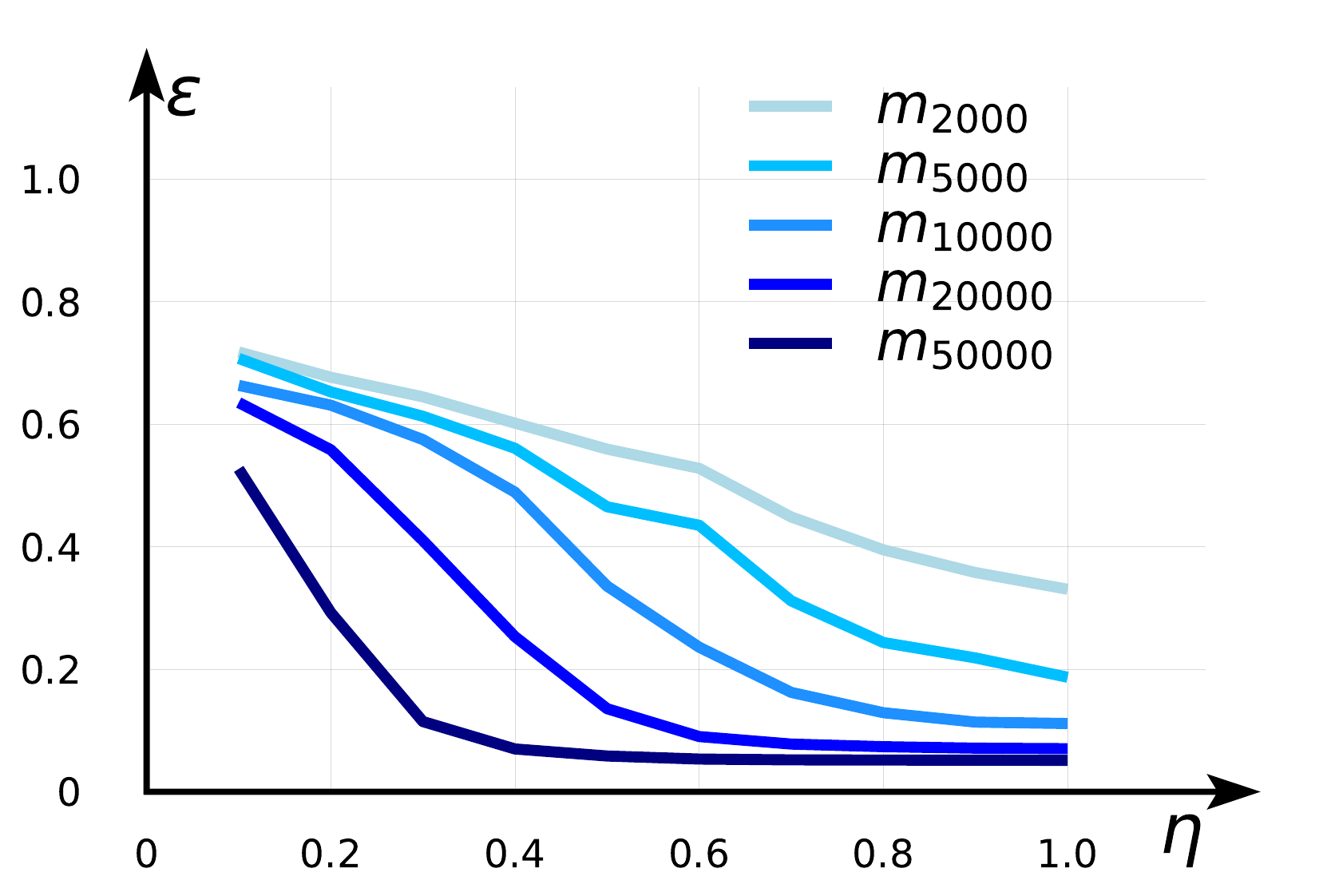}
    		\end{minipage}
		\label{figure:cifar10_complexity_curve}   
    	}
\subfigure[$\eta$-$\epsilon$ curves on CIFAR-100]{
\begin{minipage}[b]{0.3\textwidth}
    		\includegraphics[width=1.\columnwidth]{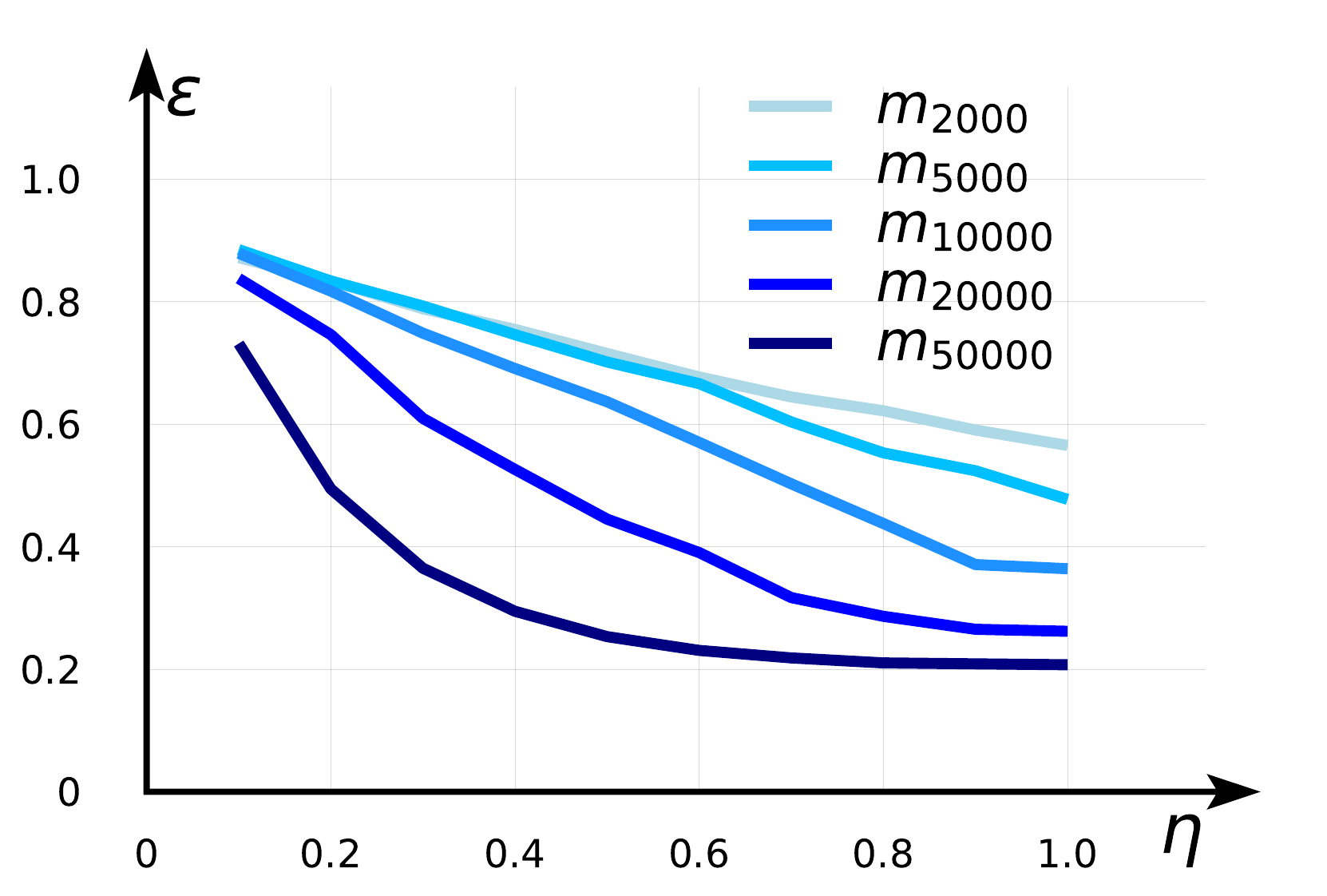}
    		\end{minipage}
		\label{figure:cifar100_complexity_curve}   
    	}
\subfigure[Schematic diagram]{
\begin{minipage}[b]{0.3\textwidth}
    		\includegraphics[width=1.\columnwidth]{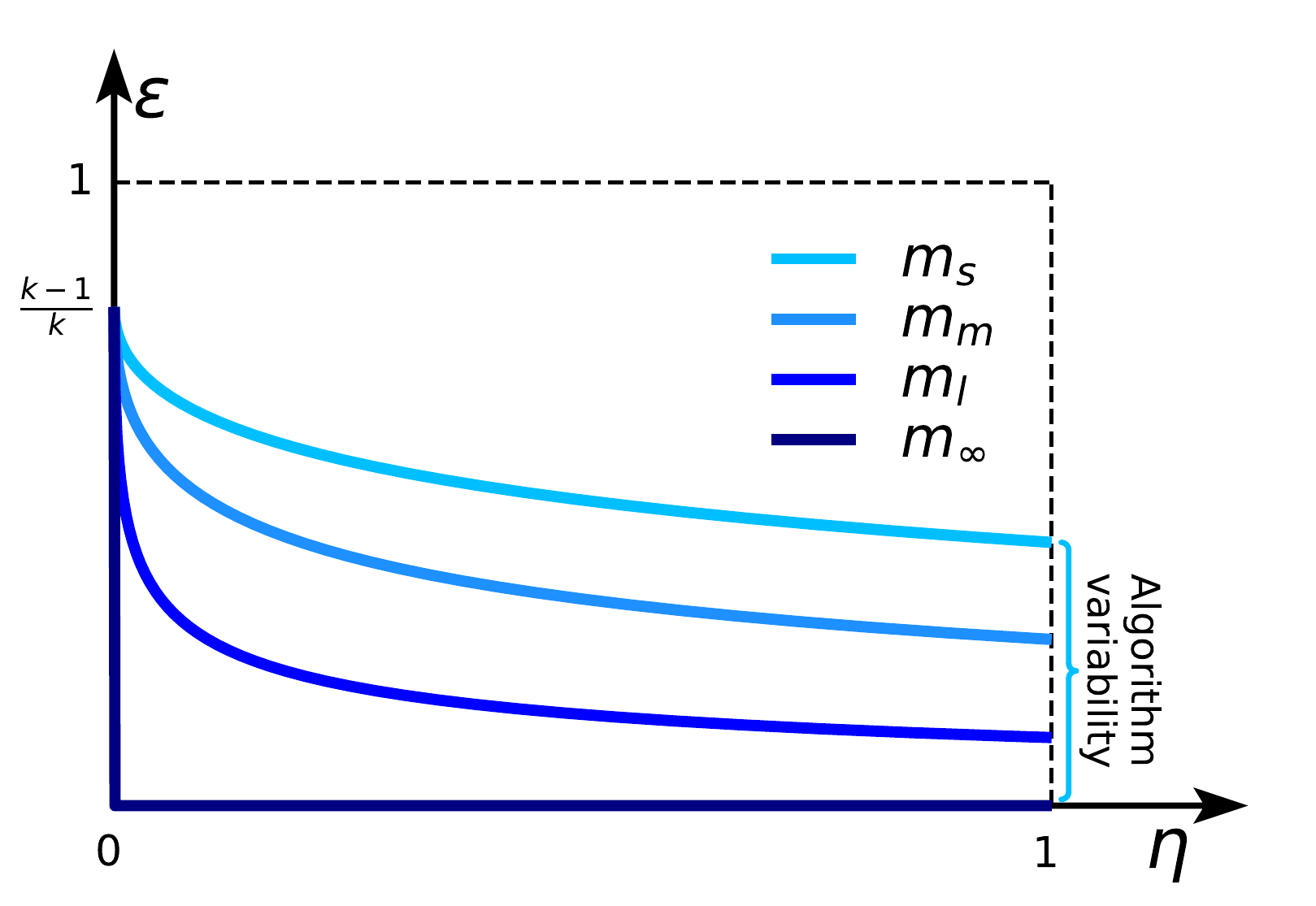}
    		\end{minipage}
		\label{figure:idea_complexity curve}   
    	}
\caption{(a) The $\eta$-$\epsilon$ curves on CIFAR-10 with different training sample sizes $2000$ ($m_{2000}$), $2000$ ($m_{2000}$), $10000$ ($m_{10000}$), $20000$ ($m_{20000}$), and $50000$ ($m_{50000}$), respectively. (b) The $\eta$-$\epsilon$ curves on CIFAR-100 with different training sample sizes. (c) The schematic diagram of the $\eta$-$\epsilon$ curves {\it w.r.t.} small ($m_s$), medium ($m_m$), large ($m_l$), and infinite ($m_\infty$) sample sizes, respectively.}
\label{figure:complexity_curve}
\end{figure*}

An illustration of data DB variability is presented in Figure \ref{figure:illustration data DB}. 
The data decision boundary variability contains two parameters of $\epsilon$ and $\eta$, respectively. That Gibbs classifier $f_{\mathcal{A}(\mathcal{S})}$ has a $(\epsilon,\eta)$-data DB variability means that only the proportion of $\eta$ of $\mathcal{S}$, {\it i.e.}, $\mathcal{S}_\eta$ (which can be considered as ``support vector set'') is enough to reconstruct a similar decision boundary with the reconstruction error $\epsilon$. The data DB variability can also be connected with the complexity of decision boundaries if we assume that simpler decision boundaries rely on a smaller number of ``support vectors''; we provide a detailed discussion in Section \ref{app:complexity of db}.

\subsection{\texorpdfstring{$\eta$-$\epsilon$ c}curves about data DB variability}
\label{sec:consistency and complexity}

According to Definition \ref{def:complexity of db}, the data DB variability degrades to the algorithm DB variability $AV(f_\mathbb{Q},\mathcal{D})$ when $\mathcal{S}_\eta=\mathcal{S}$. In other words, the algorithm DB variability is a special case of the data DB variability with $\eta=1$. 
Therefore, the data DB variability could present more detailed information on reflecting how the decision boundary variability depends on the training set, especially when we observe the variation of the reconstruction error $\epsilon$ {\it w.r.t.} different $\eta$.

To explore the relationship between the reconstruction error $\epsilon$ and the proportion of subset $\eta$,
we train $1,000$ networks of ResNet-18 on CIFAR-10 and CIFAR-100 of different sample sizes $m$. Albeit finding the most suitable $\eta$-subset is intractable, we adopt a coreset selection approach named, {\it selection via proxy} \citep{coleman2020selection}, which can rank the importance of training examples, to estimate the $\eta$-subset for a given training set $\mathcal{S}$ and proportion $\eta$. 
Then, through repeatedly training the network on $\mathcal{S}_\eta$, we can estimate the reconstruction error $\epsilon$. The $\eta$-$\epsilon$ curves of CIFAR-10 and CIFAR-100 are presented in Figure \ref{figure:cifar10_complexity_curve} and \ref{figure:cifar100_complexity_curve}, respectively. From the plots, we have an observation that {\it there is a more rapid decline in $\epsilon$ along with small $\eta$} and also a smaller algorithm DB variability when the training sample size $m$ is larger. Furthermore, we plot the schematic diagram of $\eta$-$\epsilon$ curves {\it w.r.t.} different sample size $m$, as shown in Figure \ref{figure:idea_complexity curve}. When $\eta=0$, $f_{\mathcal{A}(\mathcal{S}_n)}$ cannot be better than random guess, and hence $\epsilon=\frac{k-1}{k}$, where $k$ is the number of potential categories. It is worth noting that that $\epsilon$ has a sharper drop along with $\eta$ when the sample size $m$ is larger. Therefore, we rationally propose the following assumption, which is also shown by the right angle with $m_\infty$ in Figure \ref{figure:idea_complexity curve}.
\begin{assumption}
\label{assumption:coverge}
If $m\rightarrow \infty$, we have $\epsilon\rightarrow 0$ when $\eta\rightarrow 0$.
\end{assumption}

These plots indicate that the area under the $\eta$-$\epsilon$ curve could be a more meticulous predictor for the generalization ability of neural networks compared to the algorithm DB variability, which is only a point on the $\eta$-$\epsilon$ curve when $\eta=1$. Hence, the area under the $\eta$-$\epsilon$ curve can also be considered as an extension of the algorithm DB variability: if the Gibbs classifier $f_{\mathcal{A}(\mathcal{S})}$ possesses a smaller area under the $\eta$-$\epsilon$ curve, it produces more stable decision boundaries with varying training subsets. 

\subsection{Theoretical evidence}
\label{sec:theoretical evidence about dataset variability}

In this section, we develop the theoretical foundations for the data decision boundary variability. Our theory suggests that {\it neural networks with better data DB variability possess better generalization}, which fully supports our theory. 

According to the definition of data decision boundary variability, the $\eta$-subset $\mathcal{S}_\eta$ plays a similar role of ``support vector set'', and the complement set $\mathcal{S}\backslash\mathcal{S}_\eta = \mathcal{S} - \mathcal{S}_\eta$ is supposed to be sampled from simpler distributions than $\mathcal{D}$. Therefore, we make the following mild assumption.

\begin{assumption}
\label{assumption:data db variabilty}
Examples in $\mathcal{S}\backslash\mathcal{S}_\eta$ are assumed to be drawn from the distribution $\mathcal{D}_1$, where for all $(\mathbf{x},y)\sim \mathcal{D}_1$, $\mathbb{E}_{\boldsymbol{\theta}\sim \mathcal{A}(\mathcal{S}_\eta)}[\mathbb{I}(y\in T(f_{\boldsymbol{\theta}},\mathbf{x}))]=\max_{i\in [k]}\mathbb{E}_{\boldsymbol{\theta}\sim \mathcal{A}(\mathcal{S}_\eta)}[\mathbb{I}(i\in T(f_{\boldsymbol{\theta}},\mathbf{x}))]$ holds, and
\begin{align}
    &\mathbb{E}_{\mathcal{D}_1} \mathbb{E}_{\boldsymbol{\theta}\sim \mathcal{A}(\mathcal{S}), \boldsymbol{\theta}^\prime \sim \mathcal{A}(\mathcal{S}_\eta)} \left[\mathbb{I}\left(T(f_{\boldsymbol{\theta}}, \mathbf{x}) \neq T(f_{\boldsymbol{\theta}^\prime}, \mathbf{x})\right)\right] \nonumber \\
    & \leq  \mathbb{E}_{\mathcal{D}} \mathbb{E}_{\boldsymbol{\theta}\sim \mathcal{A}(\mathcal{S}), \boldsymbol{\theta}^\prime \sim \mathcal{A}(\mathcal{S}_\eta)} \left[\mathbb{I}\left(T(f_{\boldsymbol{\theta}}, \mathbf{x}) \neq T(f_{\boldsymbol{\theta}^\prime}, \mathbf{x})\right)\right]=\epsilon . 
\end{align}

\end{assumption}
\begin{remark}
Assumption \ref{assumption:data db variabilty} can also be stated as follows: the data ($\mathcal{D}_1$) correctly classified by $f_{\mathcal{A}(\mathcal{S}_\eta)}$ possesses a lower data decision boundary variability than the average data decision boundary variability on $\mathcal{D}$.
\end{remark}

Given Assumption \ref{assumption:data db variabilty}, we can further derive a probably approximately correct bound {\it w.r.t.} $\mathcal{R}_{\mathcal{S}\backslash\mathcal{S}_\eta}(\mathcal{A}(\mathcal{S}_\eta))$, as shown in the following lemma.

\begin{lemma}
\label{lemma:complement set bound}
If the decision boundaries of $f_{\mathcal{A}(\mathcal{S})}$ possess a  $(\epsilon,\eta)$-data DB variability and Assumption \ref{assumption:data db variabilty} holds, then, with the probability of at least $1-\delta$ over a sample of size $m$, we have
\begin{equation}
    \mathcal{R}_{\mathcal{S}\backslash\mathcal{S}_\eta}(\mathcal{A}(\mathcal{S}_\eta)) \leq \epsilon + \sqrt{\frac{1}{2(1-\eta)m}\log\frac{1}{\delta}} . 
\end{equation}
\end{lemma}

\begin{figure}[t]
\centering
\subfigure{
\begin{minipage}[b]{0.22\textwidth}
    		\includegraphics[width=0.95\columnwidth]{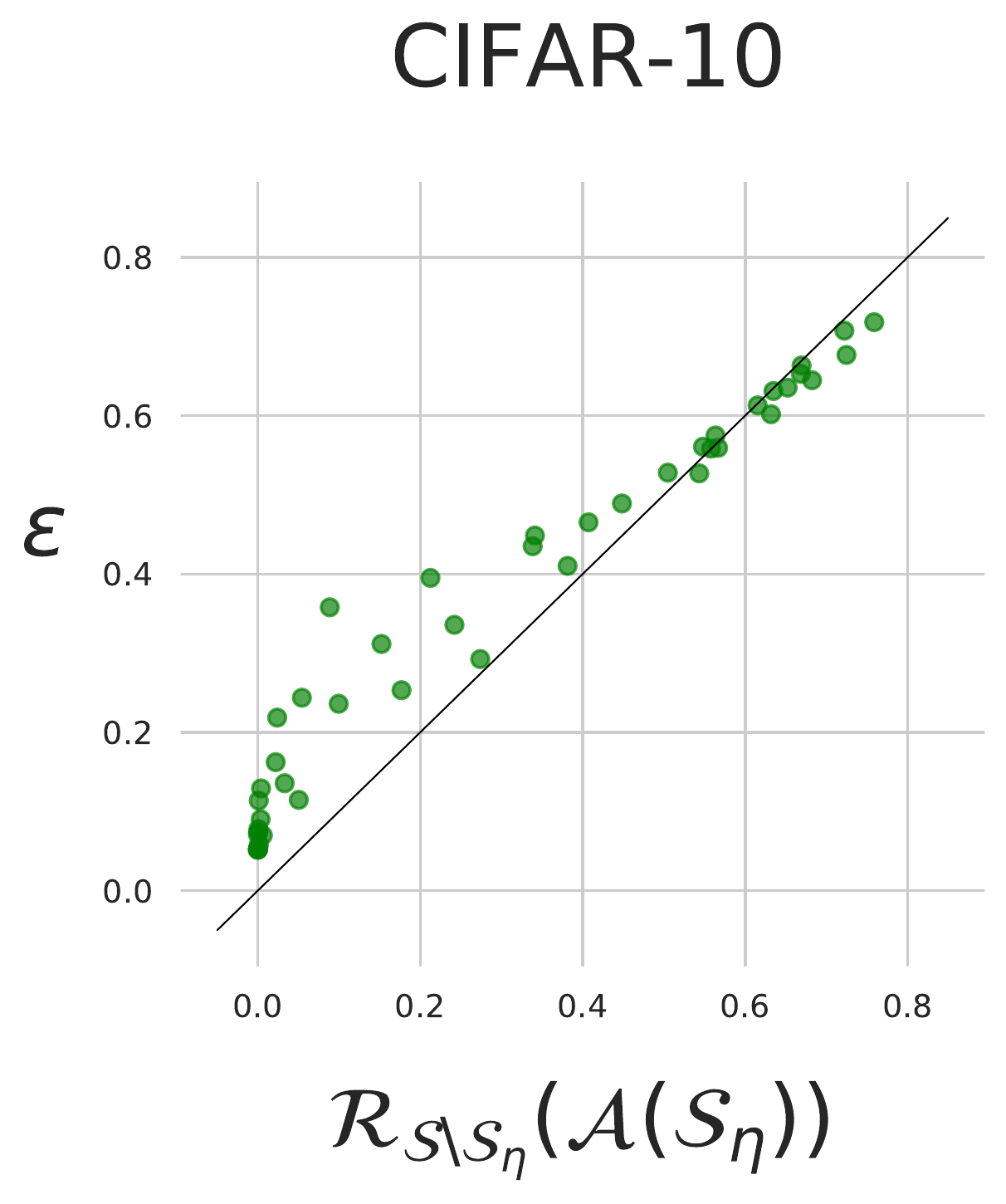}
    		\end{minipage}
		\label{figure:coreset_complement_acc_cifar10}
    	}
\subfigure{
\begin{minipage}[b]{0.22\textwidth}
    		\includegraphics[width=0.95\columnwidth]{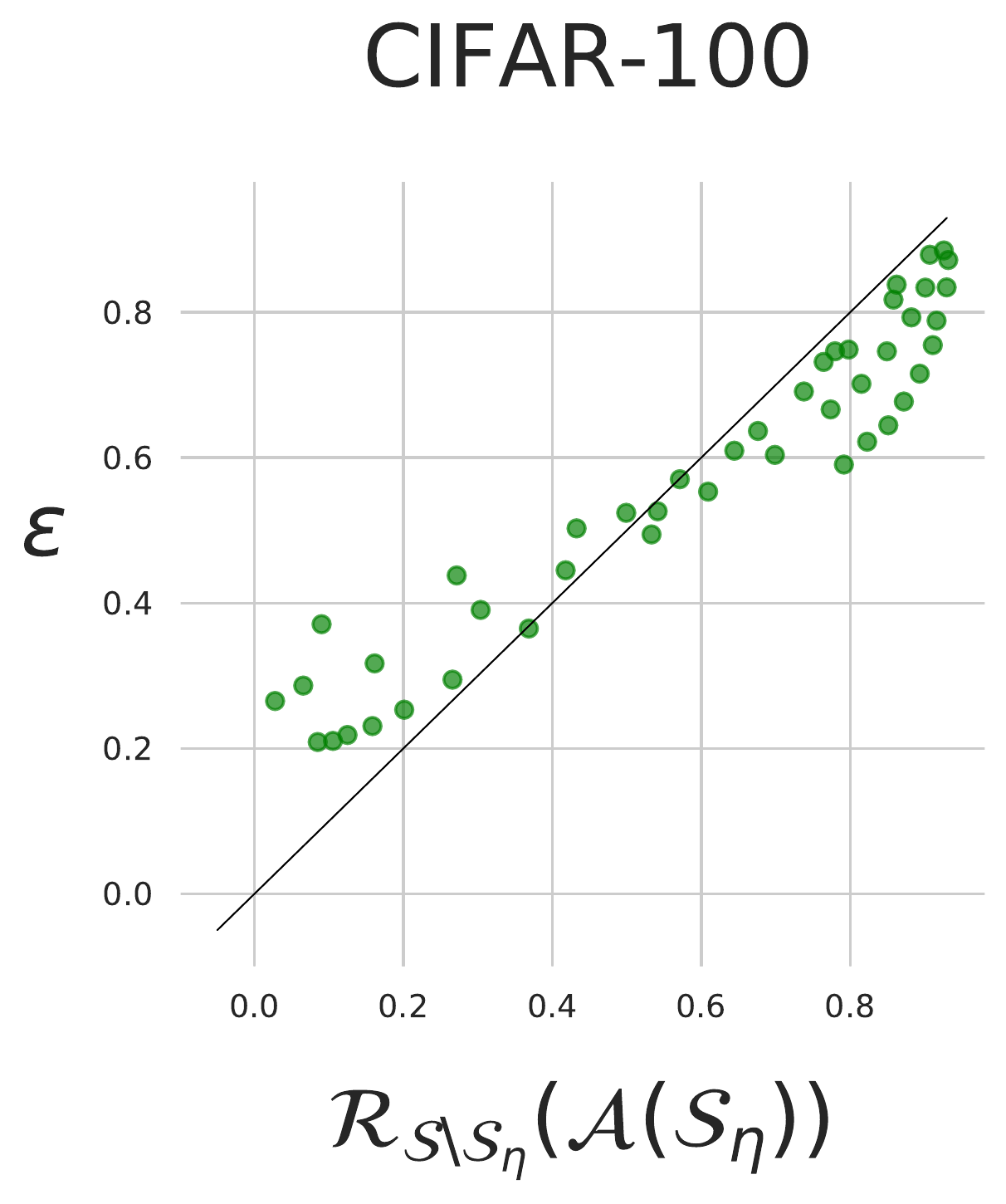}
    		\end{minipage}
		\label{figure:coreset_complement_acc_cifar100}   
    	}
\caption{Scatter of $\epsilon$ ($y$-axis) and $\mathcal{R}_{\mathcal{S}\backslash\mathcal{S}_\eta}(\mathcal{A}(\mathcal{S}_\eta))$ ($x$-axis) on CIFAR-10 and CIFAR-100.}
\label{figure:coreset_complement_acc}
\end{figure}

We also conduct experiments to show the correlation between $\mathcal{R}_{\mathcal{S}\backslash\mathcal{S}_\eta}(\mathcal{A}(\mathcal{S}_\eta))$ and $\epsilon$; see Figure \ref{figure:coreset_complement_acc}. From the plots we obtain an observation that $\mathcal{R}_{\mathcal{S}\backslash\mathcal{S}_\eta}(\mathcal{A}(\mathcal{S}_\eta))\leq \epsilon$ is stable when $\mathcal{R}_{\mathcal{S}\backslash\mathcal{S}_\eta}(\mathcal{A}(\mathcal{S}_\eta))$ is small (about less than $0.5$).

\begin{lemma}
\label{lemma:risk diff}
If the decision boundaries of $f_{\mathcal{A}(\mathcal{S})}$ possess a $(\epsilon, \eta)$-data DB variability, then we have
\begin{equation}
    \left| \mathcal{R}_\mathcal{D}(\mathcal{A}(\mathcal{S})) - \mathcal{R}_\mathcal{D}(\mathcal{A}(\mathcal{S}_\eta)) \right| \leq \epsilon. 
\end{equation}
\end{lemma}

Lemma \ref{lemma:risk diff} shows that the difference between the expected risk of $\mathcal{A}(\mathcal{S})$ and $\mathcal{A}(\mathcal{S}_\eta)$ can be bounded by their difference in decision boundaries. Then, we continue to prove the generalization bound with data decision boundary variability.

\begin{theorem}[data DB variability-based upper bound on expected risk]
\label{thm:core}
If the decision boundaries of $f_{\mathcal{A}(\mathcal{S})}$ possess a $(\epsilon, \eta)$-data DB variability on the data generating distribution $\mathcal{D}$, and assume $\eta\leq 0.5$ and Assumption \ref{assumption:data db variabilty} holds, then, with the probability of at least $1-\delta$ over a sample of size $m$, we have
\begin{equation}
\label{eq:core}
    \mathcal{R}_\mathcal{D}(\mathcal{A}(\mathcal{S})) \leq \Omega + \sqrt{4\Omega\Delta} + 8\Delta + \epsilon,
\end{equation}
where
\begin{equation}
    \Omega=\epsilon + \sqrt{\frac{1}{2(1-\eta)m}\log\frac{1}{\delta}}, 
\end{equation}
\begin{equation}
    \Delta=\eta\log\frac{e}{\eta}+\frac{1}{m}\log\frac{2}{\delta}. 
\end{equation}
Moreover, for sufficient large $m$, we have
\begin{equation}
\label{eq: dataset variability complexity}
    \mathcal{R}_\mathcal{D}(\mathcal{A}(\mathcal{S})) \leq \mathcal{O}\left(\frac{1}{\sqrt{m}}+\epsilon+\eta\log\frac{1}{\eta}\right).
\end{equation}
\end{theorem}

According to Assumption \ref{assumption:coverge}, when $m\rightarrow \infty$, $\eta\rightarrow 0$ and $\epsilon\rightarrow 0$, while according to Eq. \ref{eq: dataset variability complexity}, $\mathcal{R}_\mathcal{D}(\mathcal{A}(\mathcal{S}))\rightarrow 0$. Therefore, the generalization bound is asymptotically converged. Theorem \ref{thm:core} suggests that a smaller data DB variability, 
corresponds to a tighter upper bound on the expected risk, which theoretically verifies the relationship between the data DB variability and the generalization ability of neural networks. 

\section{Discussions}

This sections discusses how our findings would shed light on understanding other interesting phenomena.

\subsection{Algorithm DB variability and the entropy of decision boundaries}
\label{app:entropy}
If Assumption $\ref{assumption:null boundary}$ holds for all $\boldsymbol{\theta}\sim\mathbb{Q}$, $1-AV(f_\mathbb{Q},\mathcal{D})$ can be rewritten as
\begin{equation}
\mathbb{E}_{(\mathbf{x},y)\sim \mathcal{D}}\sum_{i=1}^k\mathbb{E}_{\boldsymbol{\theta}\sim \mathbb{Q}}^2\left[\mathbb{I}\left(T\left(f_{\boldsymbol{\theta}}, \mathbf{x}\right)=i\right)\right]. 
\end{equation}
The term $\sum_{i=1}^k\mathbb{E}_{\boldsymbol{\theta}\sim \mathbb{Q}}^2\left[\mathbb{I}(T\left(f_{\boldsymbol{\theta}}, \mathbf{x}\right)=i)\right]$ can be considered to measure the degree of prediction uncertainty for the given voxel $\mathbf{x}$ in the input space $\mathbb{R}^n$. If we leverage $-\log(\cdot)$ on the term $\sum_{i=1}^k\mathbb{E}_{\boldsymbol{\theta}\sim \mathbb{Q}}^2\left[\mathbb{I}(T\left(f_{\boldsymbol{\theta}}, \mathbf{x}\right)=i)\right]$, we have that
\begin{equation}
    -\log \sum_{i=1}^k\mathbb{E}_{\boldsymbol{\theta}\sim \mathbb{Q}}^2\left[\mathbb{I}(T\left(f_{\boldsymbol{\theta}}, \mathbf{x}\right)=i)\right],
\end{equation}
denotes the collision entropy of prediction made by the Gibbs classifier $f_\mathbb{Q}$ on $\mathbf{x}$. We can also replace the collision entropy with canonical Shannon entropy,
\begin{equation}
    -\sum_{i=1}^k \mathbb{E}_{\boldsymbol{\theta}\sim \mathbb{Q}}\left[\mathbb{I}(T\left(f_{\boldsymbol{\theta}}, \mathbf{x}\right)=i)\right]\log \mathbb{E}_{\boldsymbol{\theta}\sim \mathbb{Q}}\left[\mathbb{I}(T\left(f_{\boldsymbol{\theta}}, \mathbf{x}\right)=i)\right],
\end{equation}
in the future research. As such, the algorithm DB variability is closely related to the ``entropy of decision boundary'', and the uncanny generalization in neural networks might be further uncovered by investigating this low entropy of decision boundary.

\subsection{Data DB variability and the complexity of DBs}
\label{app:complexity of db}

According to the work by \citet{guan2020analysis}, a complex decision boundary has large curvatures and conjectured to indicate inferior generalization. Nevertheless, from the perspective of causality, we argue that the large curvatures or the non-linearity of DBs is {the result of hard classification tasks, not the cause}, and the primary factor in shaping a complex DB during the training procedure should be the significant non-linearity of the training data. If only the geometric properties of decision boundaries are analysed without investigating the data, the results might be incomplete and even misleading. Another obstacle for describing the complexity of DBs by its geometric properties is the huge dimensional input space, which makes the geometric properties of DBs hard to quantify and estimate. Therefore, defining the complexity of DBs based on its curvature is not rational and impractical.

Here, we consider the complexity of DBs from the perspective of the training set. During the training procedure, a small part of training examples, considered as ``support vectors'', play a more critical role in supervising the formation of decision boundaries and compelling the DB to be gradually more complicated. If the construction of decision boundaries relies on fewer ``support vectors'', the decision boundary should be simpler. In other words, if these ``support vectors'' are excluded from the training sample, the DB will be notably dissimilar when the network is retrained on the modified training set. Hence, the complexity of DBs can be also defined with the notion of the data DB variability: {\it with $(\epsilon,\eta)$-data decision boundary variability, the decision boundary of $f_{\mathcal{A}(\mathcal{S})}$ is said to possess a $(\epsilon, \eta)\textit{-complexity}$}.

By considering the data DB variability as the complexity of DB, 
many phenomenons {\it w.r.t.} generalization in deep learning can be easily understood: (1) difficult tasks generally have more complex decision boundaries, since their datum are more non-linear and contain more ``support vectors''; (2) in adversarial training, each data point is converted to a ``data ball'' with the radius of the adversarial perturbation and has more impact on forming the DBs. Hence, adversarial training contributes to a more complex decision boundary by enlarging the ``support vector set'', and thus causes the decline in generalization performance; (3) for data augmentation, generated images are also considered to obey the data generation distribution $\mathcal{D}$. Hence, data augmentation decreases the complexity of decision boundaries by greatly expanding the training set $\mathcal{S}$, while $|\mathcal{S}_\eta|$ has only a slight growth. 



\section{Experimental Implementation Details}

This section provides all the additional implementation details for our experiments.

\subsection{Model training}
\label{sec:model training}
We employ SGD to optimize all the models and the momentum factor is $0.9$. The weight decay factor is set to $5$e-$4$, and the learning rate is decayed by $0.2$ every $50$ epochs. Besides, basic data augmentation (crop and flip) \citep{Zagoruyko2016WRN} is adopted in both standard and adversarial training, and only the basic data augmentation is considered in our experiments and analysis. 

\textbf{Additional details in \ref{sec:real data}}
We train VGG-16, ResNet-18, Wide-ResNet-28 on CIFAR-10 and CIFAR-100. In the training procedure, the model is trained for $200$ epochs, in which the batch size is set to $128$, and the learning rate is initialized as $0.1$. There are three training strategies included in this experiment: standard training, non-data-augmentation training, and adversarial training. In the adversarial training, the radius of the adversarial perturbation is set as $10/255$ and $l_\infty$ distance is selected. The basic data augmentation (cropping and flipping) in the standard training and adversarial training is achieved by the following Pytorch code:
\begin{lstlisting}[language=Python]
transforms.RandomCrop(32, padding=4)
transforms.RandomHorizontalFlip()
\end{lstlisting}
The experiment is repeated for $10$ trials for each (dataset, architecture, training strategy) setting.

\textbf{Additional details in \ref{sec:training process}}
We repeatedly train $10$ ResNet-18 on CIFAR-10 and CIFAR-100, respectively, with different random seeds. In the training procedure, the model is trained for $200$ epochs, in which the batch size is set to $128$, and the learning rate is initialized as $0.1$ and $0.01$, respectively. Basic data augmentation is included during the training process.

\textbf{Additional details in \ref{sec:sample size}}
We randomly sample examples from the training set of CIFAR-10 and CIFAR-100 to form five datasets with different sizes of $[2000, 5000, 10000, 20000, 50000]$, respectively. $10$ ResNet-18 are trained for each dataset. In the training procedure, the model is trained for $200$ epochs, in which the batch size is set to $128$, and the learning rate is initialized as $0.1$. Basic data augmentation is included during the training process.

\textbf{Additional details in \ref{sec:label noise}}
We randomly change the labels of $20\%$ examples in the training set of CIFAR-10 and CIFAR-100. Then, $10$ ResNet-18 are optimize by SGD for $500$ epochs on the noise CIFAR-10 and CIFAR-100, respectively. the momentum factor is $0.9$, and the learning rate is $0.001$ and does not decay during the training process.

{

\textbf{Additional details in \ref{sec:model selection}} We conduct the model selection experiments with convolutional (CNN) and residual (ResNet) networks for CIFAR-10/100 from \citet{lotfi2022bayesian}. We use the same architectures:
\begin{itemize}
    \item The CNNs consist of up to $5$ blocks of $33$ convolutions, followed by a ReLU activation function, and MaxPooling, except in the first layer. BatchNorm is replaced by the fixup initialization \citep{zhang2019fixup} as in \citet{immer2021scalable}. The width of the first channel varies from $2$ to $32$ for both datasets. The last layer is a fully-connected layer to the class logit.
    \item ResNets of depths varying from $8$ to $32$ are used for CIFAR-10 and from $20$ to $101$ for CIFAR-100. The width varies from $16$ to $48$ for CIFAR-10 and from $32$ to $64$ for CIFAR-100.
\end{itemize}

All models are trained for $100$ epochs with initialised learning rate $0.01$, and other settings are the same as them in \ref{sec:model training} including SGD optimizer, learning rate decay, and data augmentation. We used the Kronecker Laplace approximation to compute the averaging LML over 55 repeated trainings.
}

\textbf{Additional details in \ref{sec:consistency and complexity}}
We randomly sample examples from the training set of CIFAR-10 and CIFAR-100 to form five datasets with different sizes of $[2000,5000,10000,20000,50000]$, respectively. For each dataset, we obtain 1010 $\eta$-subsets with different $\eta$ of $[0.1,0.2,0.3,0.4,0.5,0.6,0.7,0.8,0.9,1.0]$ via a coreset selection approach named {\it selection via proxy} \citep{coleman2020selection}. The related code can be downloaded from \url{https://github.com/stanford-futuredata/selection-via-proxy}. The ResNet-18 is repeatedly trained for $10$ trials to estimate the complexity of decision boundaries for each $\eta$-subset.

\section{Proofs}
The section collects detailed proofs of the results that are omitted in Section \ref{sec:theoretical evidence about algorithm variability} and \ref{sec:theoretical evidence about dataset variability}. To avoid technicalities, the measurability/integrability issues are ignored throughout this paper. Moreover, Fubini's theorem is assumed to be applicable for any integration {\it w.r.t.} multiple variables. In other words, the order of integrations is exchangeable.

\subsection{Proof of Theorem \ref{thm:lower bound}}
\label{app:proof of lower bound}
\begin{proof} If Assumption \ref{assumption:null boundary} holds for all $\boldsymbol{\theta}\sim \mathbb{Q}$, for any networks $f_{\boldsymbol{\theta}}$ and $f_{\boldsymbol{\theta}^\prime}$ we have
{\scriptsize
$$
    \mathbb{E}_{(\mathbf{x},y)\sim \mathcal{D}}\left[\mathbb{I}\left(y\in T\left(f_{\boldsymbol{\theta}}, \mathbf{x}\right)\right)\mathbb{I}\left(y\in T\left(f_{\boldsymbol{\theta}^\prime}, \mathbf{x}\right)\right) \mathbb{I}\left(T\left(f_{\boldsymbol{\theta}}, \mathbf{x}\right) \neq T\left(f_{\boldsymbol{\theta}^\prime}, \mathbf{x}\right)\right) \right] = 0
$$
}
Hence,
\blue{
\begin{align}
    & AV(f_\mathbb{Q},\mathcal{D}) \nonumber\\
    &=\mathbb{E}_{(\mathbf{x},y)\sim \mathcal{D}}\mathbb{E}_{\boldsymbol{\theta},\boldsymbol{\theta}^\prime \sim \mathbb{Q}}\left[\mathbb{I}\left(T(f_{\boldsymbol{\theta}}, \mathbf{x}) \neq T(f_{\boldsymbol{\theta}^\prime}, \mathbf{x})\right)\right] \nonumber\\
    &=\mathbb{E}_{(\mathbf{x},y)\sim \mathcal{D}}\mathbb{E}_{\boldsymbol{\theta},\boldsymbol{\theta}^\prime \sim \mathbb{Q}}\left[\mathbb{I}\left(y\in T(f_{\boldsymbol{\theta}}, \mathbf{x})\right)\mathbb{I}\left(y\notin T(f_{\boldsymbol{\theta}^\prime}, \mathbf{x})\right)\right.\nonumber \\
    &\quad+\left. \mathbb{I}\left(y\notin T(f_{\boldsymbol{\theta}}, \mathbf{x})\right)\mathbb{I}\left(T(f_{\boldsymbol{\theta}}, \mathbf{x}) \neq T(f_{\boldsymbol{\theta}^\prime}, \mathbf{x})\right) \right] \nonumber\\
    &\leq \mathbb{E}_{(\mathbf{x},y)\sim \mathcal{D}}\mathbb{E}_{\boldsymbol{\theta},\boldsymbol{\theta}^\prime \sim \mathbb{Q}}\left[\mathbb{I}\left(y\in T(f_{\boldsymbol{\theta}}, \mathbf{x})\right)\mathbb{I}\left(y\notin T(f_{\boldsymbol{\theta}^\prime}, \mathbf{x})\right)\right] + \mathcal{R}_{\mathcal{D}}(\mathbb{Q})  \nonumber\\
    &=2\mathcal{R}_{\mathcal{D}}(\mathbb{Q})  - \mathbb{E}_{(\mathbf{x},y)\sim \mathcal{D}}\mathbb{E}_{\boldsymbol{\theta},\boldsymbol{\theta}^\prime \sim \mathbb{Q}}\left[\mathbb{I}\left(y\notin T(f_{\boldsymbol{\theta}}, \mathbf{x})\right)\mathbb{I}\left(y\notin T(f_{\boldsymbol{\theta}^\prime}, \mathbf{x})\right)\right] \nonumber\\
    &= 2\mathcal{R}_{\mathcal{D}}(\mathbb{Q}) - \mathbb{E}_{(\mathbf{x},y)\sim \mathcal{D}}\left[\mathbb{E}_{\boldsymbol{\theta} \sim \mathbb{Q}}^2\left[\mathbb{I}\left(y\notin T(f_{\boldsymbol{\theta}}, \mathbf{x})\right)\right]\right] \nonumber \\
    &\leq 2\mathcal{R}_{\mathcal{D}}(\mathbb{Q}) - \left[\mathbb{E}_{(\mathbf{x},y)\sim \mathcal{D}}\mathbb{E}_{\boldsymbol{\theta} \sim \mathbb{Q}}\left[\mathbb{I}\left(y\notin T(f_{\boldsymbol{\theta}}, \mathbf{x})\right)\right]\right]^2 \nonumber \\
    &\leq 2\mathcal{R}_{\mathcal{D}}(\mathbb{Q}) - \mathcal{R}_{\mathcal{D}}^2(\mathbb{Q})  \nonumber
\end{align}
Solving the inequality yields the desired bound and finishes the proof.
}
\end{proof}

\subsection{Proof of Theorem \ref{thm:lower bound for binary case}}
\begin{proof}
When the classification is binary, {\it i.e.}, $k=2$, with Assumption \ref{assumption:null boundary}, we have
\begin{align}
    & \mathbb{E}_{(\mathbf{x},y)\sim \mathcal{D}}\mathbb{E}_{\boldsymbol{\theta},\boldsymbol{\theta}^\prime \sim \mathbb{Q}}\left[\mathbb{I}\left(T\left(f_{\boldsymbol{\theta}}, \mathbf{x}\right) = T\left(f_{\boldsymbol{\theta}}, \mathbf{x}\right)\right)\right] \nonumber\\
    =& \mathbb{E}_{(\mathbf{x},y)\sim \mathcal{D}}\mathbb{E}_{\boldsymbol{\theta}\sim \mathbb{Q}}^2\left[\mathbb{I}\left(y\in T\left(f_{\boldsymbol{\theta}}, \mathbf{x}\right)\right)\right] \nonumber \nonumber\\
    &+ \mathbb{E}_{(\mathbf{x},y)\sim \mathcal{D}}\mathbb{E}_{\boldsymbol{\theta}\sim \mathbb{Q}}^2\left[\mathbb{I}\left(y\notin  T\left(f_{\boldsymbol{\theta}}, \mathbf{x}\right)\right)\right] \nonumber\\
    =& \mathbb{E}_{(\mathbf{x},y)\sim \mathcal{D}}\mathbb{E}_{\boldsymbol{\theta}\sim \mathbb{Q}}^2\left[\mathbb{I}\left(y\in T\left(f_{\boldsymbol{\theta}}, \mathbf{x}\right)\right)\right] \nonumber \nonumber\\
    &+ \mathbb{E}_{(\mathbf{x},y)\sim \mathcal{D}}\left[1- \mathbb{E}_{\boldsymbol{\theta}\sim \mathbb{Q}}\left[\mathbb{I}\left(y\in T\left(f_{\boldsymbol{\theta}}, \mathbf{x}\right)\right)\right]\right]^2 \nonumber\\
    =& 2\mathbb{E}_{(\mathbf{x},y)\sim \mathcal{D}}\mathbb{E}_{\boldsymbol{\theta}\sim \mathbb{Q}}^2\left[\mathbb{I}\left(y\in T\left(f_{\boldsymbol{\theta}}, \mathbf{x}\right)\right)\right] + 2\mathcal{R}_\mathcal{D}(\mathbb{Q}) - 1 \nonumber
\end{align}

Plugging in $\operatorname{Var}_{(\mathbf{x},y)\sim \mathcal{D}}\left[\mathbb{E}_{\boldsymbol{\theta}\sim \mathbb{Q}}\left[\mathbb{I}\left(y\in T\left(f_{\boldsymbol{\theta}}, \mathbf{x}\right)\right)\right]\right] = \mathbb{E}_{(\mathbf{x},y)\sim \mathcal{D}}\mathbb{E}_{\boldsymbol{\theta}\sim \mathbb{Q}}^2\left[\mathbb{I}\left(y\in T\left(f_{\boldsymbol{\theta}}, \mathbf{x}\right)\right)\right] - (1-\mathcal{R}_{\mathcal{D}}(\mathbb{Q}))^2$ yields
\begin{align}
    &\mathbb{E}_{(\mathbf{x},y)\sim \mathcal{D}}\mathbb{E}_{\boldsymbol{\theta},\boldsymbol{\theta}^\prime \sim \mathbb{Q}}\left[\mathbb{I}\left(T\left(f_{\boldsymbol{\theta}}, \mathbf{x}\right) = T\left(f_{\boldsymbol{\theta}}, \mathbf{x}\right)\right)\right] \nonumber\\
    =& 2\operatorname{Var}_{(\mathbf{x},y)\sim \mathcal{D}}\left[\mathbb{E}_{\boldsymbol{\theta}\sim \mathbb{Q}}\left[\mathbb{I}\left(y\in T\left(f_{\boldsymbol{\theta}}, \mathbf{x}\right)\right)\right]\right] \nonumber \\
    &+ \mathcal{R}_\mathcal{D}^2(\mathbb{Q}) + (1- \mathcal{R}_\mathcal{D}(\mathbb{Q}))^2 \nonumber\\
    \geq & \mathcal{R}_\mathcal{D}^2(\mathbb{Q}) + (1- \mathcal{R}_\mathcal{D}(\mathbb{Q}))^2 \nonumber
\end{align}
Plugging in $\mathbb{E}_{(\mathbf{x},y)\sim \mathcal{D}}\mathbb{E}_{\boldsymbol{\theta},\boldsymbol{\theta}^\prime \sim \mathbb{Q}}\left[\mathbb{I}\left(T(f_{\boldsymbol{\theta}}, \mathbf{x}) = T(f_{\boldsymbol{\theta}^\prime}, \mathbf{x})\right)\right] = 1- AV(f_\mathbb{Q}, \mathcal{D})$ and solving the inequality yield the desired inequality and finishes the proof of Theorem \ref{thm:lower bound for binary case}. 
\end{proof}

\subsection{Proof of Lemma \ref{lemma:complement set bound}}
\label{app:complement set bound}
\begin{proof}
According to Assumption \ref{assumption:data db variabilty} that for all $(\mathbf{x},y)\sim \mathcal{D}_1$, we have
\begin{small}
$$\mathbb{E}_{\boldsymbol{\theta}\sim \mathcal{A}(\mathcal{S}_\eta)}[\mathbb{I}(y\in T(f_{\boldsymbol{\theta}},\mathbf{x}))]=\max_{i\in [k]}\mathbb{E}_{\boldsymbol{\theta}\sim \mathcal{A}(\mathcal{S}_\eta)}[\mathbb{I}(i\in T(f_{\boldsymbol{\theta}},\mathbf{x}))],$$
\end{small}
then 
\begin{small}
\begin{align}
    &\mathcal{R}_{\mathcal{D}_1}(\mathcal{A}(\mathcal{S}_\eta)) \nonumber\\
    = &1 -  \mathbb{E}_{(\mathbf{x},y)\sim \mathcal{D}_1}\mathbb{E}_{\boldsymbol{\theta}\sim \mathcal{A}(\mathcal{S}_\eta)} \left[\mathbb{I}\left(y =  T\left(f_{\boldsymbol{\theta}}, \mathbf{x}\right)\right)\right] \nonumber\\
    =& 1 - \mathbb{E}_{\mathcal{D}_1}\sum_{i=1}^k\mathbb{E}_{\mathcal{A}(\mathcal{S})} \left[\mathbb{I}\left(i =  T\left(f_{\boldsymbol{\theta}}, \mathbf{x}\right)\right)\right] \mathbb{E}_{\mathcal{A}(\mathcal{S}_\eta)} \left[\mathbb{I}\left(y =  T\left(f_{\boldsymbol{\theta}}, \mathbf{x}\right)\right)\right] \nonumber\\
    \leq & 1 - \mathbb{E}_{\mathcal{D}_1}\sum_{i=1}^k\mathbb{E}_{\mathcal{A}(\mathcal{S})} \left[\mathbb{I}\left(i =  T\left(f_{\boldsymbol{\theta}}, \mathbf{x}\right)\right)\right] \mathbb{E}_{ \mathcal{A}(\mathcal{S}_\eta)} \left[\mathbb{I}\left(i =  T\left(f_{\boldsymbol{\theta}}, \mathbf{x}\right)\right)\right] \nonumber\\
    =&\mathbb{E}_{\mathcal{D}_1} \mathbb{E}_{\boldsymbol{\theta}\sim \mathcal{A}(\mathcal{S}), \boldsymbol{\theta}^\prime \sim \mathcal{A}(\mathcal{S}_\eta)} \left[\mathbb{I}\left(T(f_{\boldsymbol{\theta}}, \mathbf{x}) \neq T(f_{\boldsymbol{\theta}^\prime}, \mathbf{x})\right)\right] \nonumber\\
    \leq& \mathbb{E}_{\mathcal{D}} \mathbb{E}_{\boldsymbol{\theta}\sim \mathcal{A}(\mathcal{S}), \boldsymbol{\theta}^\prime \sim \mathcal{A}(\mathcal{S}_\eta)} \left[\mathbb{I}\left(T(f_{\boldsymbol{\theta}}, \mathbf{x}) \neq T(f_{\boldsymbol{\theta}^\prime}, \mathbf{x})\right)\right] \nonumber\\
    =&\epsilon. \nonumber
\end{align}
\end{small}
Because the examples in $\mathcal{S}\backslash\mathcal{S}_\eta$ are drawn from $\mathcal{D}_1$, by applying Hoeffding's Inequality, we have
\begin{equation}
\label{eq:hoe}
    \text{Pr}\left[ \mathcal{R}_{\mathcal{S}\backslash\mathcal{S}_\eta} - \mathcal{R}_{\mathcal{D}_1} \geq t \right] \leq \exp(-2(1-\eta)mt^2).
\end{equation}
Plug in $\delta=\exp(-2(1-\eta)mt^2)$ into Eq. \ref{eq:hoe}, thus, with the probability of at least $1-\delta$, we have
\begin{equation}
    \mathcal{R}_{\mathcal{S}\backslash\mathcal{S}_\eta}(\mathcal{A}(\mathcal{S}_\eta)) \leq \epsilon + \sqrt{\frac{1}{2(1-\eta)m}\log\frac{1}{\delta}}.  \nonumber
\end{equation}
\end{proof}

\subsection{Proof of Lemma \ref{lemma:risk diff}}
\label{app:proof lemma1}
\begin{proof}
From the definition of data DB variability, there exists a $\eta$-subset $\mathcal{S}_\eta$ s.t.
\begin{equation}
    \mathbb{E}_{(\mathbf{x},y)\sim \mathcal{D}} \mathbb{E}_{\boldsymbol{\theta}\sim \mathcal{A}(\mathcal{S}), \boldsymbol{\theta}^\prime \sim \mathcal{A}(\mathcal{S}_\eta)} \left[\mathbb{I}\left(T(f_{\boldsymbol{\theta}}, \mathbf{x}) \neq T(f_{\boldsymbol{\theta}^\prime}, \mathbf{x})\right)\right] = \epsilon. \nonumber
\end{equation}

Recall the $\text{LHS}=\left| \mathcal{R}_\mathcal{D}(\mathcal{A}(\mathcal{S})) - \mathcal{R}_\mathcal{D}(\mathcal{A}(\mathcal{S}_\eta)) \right|$ and denote $\mathbb{E}_{(\mathbf{x},y)\sim \mathcal{D}} \mathbb{E}_{\boldsymbol{\theta}\sim \mathcal{A}(\mathcal{S}), \boldsymbol{\theta}^\prime \sim \mathcal{A}(\mathcal{S}_\eta)}$ as $\mathbb{E}_{\mathcal{D}, \mathcal{A}(\mathcal{S}), \mathcal{A}(\mathcal{S}_\eta)}$ for simplicity, then
\begin{align}
    & \text{LHS}  \nonumber\\
    &=\left| \mathbb{E}_{(\mathbf{x}, y) \sim \mathcal{D}} \mathbb{E}_{\boldsymbol{\theta} \sim \mathcal{A}(\mathcal{S})} \left[\mathbb{I}\left(y\notin T(f_{\boldsymbol{\theta}}, \mathbf{x}) \right)\right] \right. \nonumber \\
    &\quad - \left.\mathbb{E}_{(\mathbf{x}, y) \sim \mathcal{D}} \mathbb{E}_{\boldsymbol{\theta} \sim \mathcal{A}(\mathcal{S}_\eta)} \left[\mathbb{I}\left(y\notin T(f_{\boldsymbol{\theta}}, \mathbf{x})\right)\right] \right| \nonumber\\
    &= \left| \mathbb{E}_{\mathcal{D}, \mathcal{A}(\mathcal{S}), \mathcal{A}(\mathcal{S}_\eta)}\left[\mathbb{I}\left(y\notin T(f_{\boldsymbol{\theta}}, \mathbf{x})\right) - \mathbb{I}\left(y\notin T(f_{\boldsymbol{\theta}^\prime}, \mathbf{x})\right)  \right] \right| \nonumber\\
    &\leq \mathbb{E}_{\mathcal{D}, \mathcal{A}(\mathcal{S}), \mathcal{A}(\mathcal{S}_\eta)}\left[\left|\mathbb{I}\left(y\notin T(f_{\boldsymbol{\theta}}, \mathbf{x})\right) - \mathbb{I}\left(y\notin T(f_{\boldsymbol{\theta}^\prime}, \mathbf{x})\right)\right|  \right] \nonumber\\
    &\leq \mathbb{E}_{\mathcal{D}, \mathcal{A}(\mathcal{S}), \mathcal{A}(\mathcal{S}_\eta)} \left[\mathbb{I}\left(T(f_{\boldsymbol{\theta}}, \mathbf{x}) \neq T(f_{\boldsymbol{\theta}^\prime}, \mathbf{x})\right)\right] = \epsilon \nonumber
\end{align}
The proof is completed.
\end{proof}

\subsection{Proof of Theorem \ref{thm:core}}
\label{app:proof core}
We first introduce Lemma \ref{lemma:30.1} and Lemma \ref{lemma:core lemma} as below.
\begin{lemma}[Lemma 30.1 in \citep{shalev2014understanding}]
\label{lemma:30.1}
Assume $T$ and $V$ are two datasets independently sampled from the data generating distribution $\mathcal{D}$, then, with the probability of at least $1-\delta$, we have
\begin{small}
\begin{equation}
    \mathcal{R}_\mathcal{D}(\mathcal{A}(T)) \leq \mathcal{R}_V(\mathcal{A}(T)) + \sqrt{\frac{2 \mathcal{R}_V(\mathcal{A}(T)) \log (1 / \delta)}{|V|}}+\frac{4 \log (1 / \delta)}{|V|}. \nonumber
\end{equation}
\end{small}

\end{lemma}

\begin{lemma}[Theorem 30.2 in \citep{shalev2014understanding}]
\label{lemma:core lemma}
Let $\mathcal{S}_\eta$ be a $\eta$-subset of the dataset $\mathcal{S}$, which is sampled from the data generation distribution $\mathcal{D}$ and the sample size $|\mathcal{S}|=m$. Let $\mathcal{S}\backslash\mathcal{S}_\eta=\mathcal{S}-\mathcal{S}_\eta$ and assume $\eta\leq 0.5$. Then, with the probability of at least $1-\delta$ over a sample of size $m$, we have
\begin{equation}
    \mathcal{R}_\mathcal{D}(\mathcal{A}(\mathcal{S}_\eta)) \leq \mathcal{R}_{\mathcal{S}\backslash\mathcal{S}_\eta}(\mathcal{A}(\mathcal{S}_\eta)) + \sqrt{4\mathcal{R}_{\mathcal{S}\backslash\mathcal{S}_\eta}(\mathcal{A}(\mathcal{S}_\eta))\Delta} + 8\Delta, \nonumber
\end{equation}
where
\begin{equation}
    \Delta=\eta\log\frac{e}{\eta}+\frac{1}{m}\log\frac{1}{\delta}. \nonumber
\end{equation}

\end{lemma}

\begin{proof}[Proof of Lemma \ref{lemma:core lemma}]
\begin{align}
    \text{Pr}&\left[\exists \mathcal{S}_\eta \subseteq \mathcal{S} \text{ s.t. } \mathcal{R}_\mathcal{D}(\mathcal{A}(\mathcal{S}_\eta)) \leq \mathcal{R}_{\mathcal{S}\backslash\mathcal{S}_\eta}(\mathcal{A}(\mathcal{S}_\eta)) \right. \nonumber\\
    & \qquad \left. + \sqrt{\frac{2 \mathcal{R}_{\mathcal{S}\backslash\mathcal{S}_\eta}(\mathcal{A}(\mathcal{S}_\eta)) \log (1 / \delta)}{|\mathcal{S}\backslash\mathcal{S}_\eta|}}+\frac{4 \log (1 / \delta)}{|\mathcal{S}\backslash\mathcal{S}_\eta|} \right] \nonumber\\
    &\leq \sum_{\mathcal{S}_\eta \subseteq \mathcal{S}} \text{Pr}\left[\mathcal{R}_\mathcal{D}(\mathcal{A}(\mathcal{S}_\eta)) \leq \mathcal{R}_{\mathcal{S}\backslash\mathcal{S}_\eta}(\mathcal{A}(\mathcal{S}_\eta)) \right. \nonumber \nonumber\\
    & \qquad \left.+ \sqrt{\frac{2 \mathcal{R}_{\mathcal{S}\backslash\mathcal{S}_\eta}(\mathcal{A}(\mathcal{S}_\eta)) \log (1 / \delta)}{|\mathcal{S}\backslash\mathcal{S}_\eta|}}+\frac{4 \log (1 / \delta)}{|\mathcal{S}\backslash\mathcal{S}_\eta|}\right] \nonumber\\
    &= {m\choose \eta m}\delta \leq \left(\frac{e}{\eta}\right)^{\eta m}\delta \nonumber
\end{align}
Plug in $\delta^\prime=\left(\frac{e}{\eta}\right)^{\eta m}\delta$, and use the assumption $\eta \leq \frac{1}{2}$, which implies $|\mathcal{S}\backslash\mathcal{S}_\eta|\geq \frac{m}{2}$, then, with the probability of at least $1-\delta^\prime$, we have that
\begin{align}
    \mathcal{R}_\mathcal{D}(\mathcal{A}&(\mathcal{S}_\eta)) \nonumber \\
    &\leq \mathcal{R}_{\mathcal{S}\backslash\mathcal{S}_\eta}(\mathcal{A}(\mathcal{S}_\eta))\nonumber \\
    &\quad + \sqrt{4\mathcal{R}_{\mathcal{S}\backslash\mathcal{S}_\eta}(\mathcal{A}(\mathcal{S}_\eta))\left(\eta\log\frac{e}{\eta} +\frac{1}{m}\log\frac{1}{\delta^\prime}\right)} \nonumber \\
    &\quad + 8\left(\eta\log\frac{e}{\eta}+\frac{1}{m}\log\frac{1}{\delta^\prime}\right), \nonumber
\end{align}
which concludes the proof.
\end{proof}

With the above lemmas, we can derive the generalization bound based on the complexity of decision boundary.

\begin{proof}[Proof of Theorem \ref{thm:core}]

According to Lemma \ref{lemma:core lemma}, with the probability of at least $1-\delta$, we have
\begin{equation}
\mathcal{R}_{\mathcal{S}\backslash\mathcal{S}_\eta}(\mathcal{A}(\mathcal{S}_\eta)) \leq \epsilon + \sqrt{\frac{1}{2(1-\eta)m}\log\frac{1}{\delta}}. \nonumber
\end{equation}

Through combining this with Lemma \ref{lemma:core lemma}, with the probability of at least $1-2\delta$, we have
\begin{equation}
\label{eq:r}
    \mathcal{R}_\mathcal{D}(\mathcal{A}(\mathcal{S}_\eta)) \leq \Omega + \sqrt{4\Omega\Delta} + 8\Delta, 
\end{equation}
where
\begin{equation}
    \Omega=\epsilon + \sqrt{\frac{1}{2(1-\eta)m}\log\frac{1}{\delta}}, \nonumber
\end{equation}
\begin{equation}
    \Delta=\eta\log\frac{e}{\eta}+\frac{1}{m}\log\frac{1}{\delta}. \nonumber
\end{equation}
Plugging the equality of $\mathcal{R}_\mathcal{D}\left(\mathcal{A}\left(\mathcal{S}\right)\right)\leq \mathcal{R}_\mathcal{D}\left(\mathcal{A}\left(\mathcal{S}_\eta\right)\right) + \epsilon$ in Lemma \ref{lemma:risk diff} into Eq. \ref{eq:r} yields the desired inequality of Eq. \ref{eq:core}.

When $m$ is sufficient large, $\sqrt{4\Omega\Delta}$ can be dropped due to $\sqrt{4\Omega\Delta}\leq \Omega + \Delta$. Considering $\eta\leq 0.5$, $\Omega\leq \sqrt{\frac{1}{2m}\log\frac{1}{\delta}} + \epsilon=\mathcal{O}(\frac{1}{\sqrt{m}}+\epsilon)$. The term $\frac{1}{m}\log\frac{1}{\delta}$ in $\Delta$ can be dropped because it has a faster convergence speed compared to $\sqrt{\frac{1}{2m}\log\frac{1}{\delta}}$ in $\Omega$. Because $\log\frac{1}{\delta}$ is considered as a constant, we have
\begin{equation}
    \mathcal{R}_\mathcal{D}(\mathcal{A}(\mathcal{S})) \leq \mathcal{O}(\frac{1}{\sqrt{m}}+\epsilon+\eta\log\frac{1}{\eta}). \nonumber
\end{equation}
The proof of Theorem \ref{thm:core} is finished.
\end{proof}

\section{Conclusion}
In this paper, we empirically and theoretically explored the relationship between decision boundary variability and generalization in neural networks, through the notions of algorithm DB variability and data DB variability, respectively. A significant negative correlation between the decision boundary variability and generalization performance is observed in our experiments. As for the theoretical results, two lower bounds based on algorithm DB variability and an upper bound based on data DB variability are proposed, respectively, for the sake of enhancing our findings.


\small
\bibliographystyle{IEEEtranN}
\bibliography{consistency}

\end{document}